\documentclass[twoside,11pt]{article}

\usepackage[preprint]{jmlr2e}
\editor{}

\usepackage{amsmath}
\usepackage{algorithm}
\usepackage{algpseudocode}
\usepackage{booktabs}
\usepackage{longtable}
\usepackage{pdflscape}
\usepackage{array}
\usepackage{xcolor} 
\usepackage[caption=false,font=small]{subfig}
\usepackage{cleveref}

\newtheorem{thm}{Theorem}
\newtheorem{lem}[thm]{Lemma}

\newtheorem{assum}[thm]{Assumption}

\ShortHeadings{Adam Improves Muon}{Zhang, Liu and Schaeffer}

\firstpageno{1}

\firstpageno{1}
\usepackage{graphicx} 
\usepackage{amsmath,amssymb,makecell,multicol}
\usepackage{booktabs}
\usepackage[letterpaper,top=1in,bottom=1in,left=1in,right=1in,marginparwidth=1.75cm]{geometry}
\usepackage{algorithm}
\usepackage{algpseudocode}
\usepackage{xcolor}
\usepackage{makecell}
\usepackage{array}
\usepackage{booktabs}
\usepackage{longtable}    
\usepackage{lscape}       
\usepackage{pdflscape}    
\usepackage{geometry}     

\usepackage{hyperref}
\usepackage{cleveref} 
\usepackage{float}                            
\usepackage[caption=false,font=small]{subfig} 
\usepackage{multirow}

\DeclareMathOperator{\diag}{diag}

\usepackage[parfill]{parskip}

\title{Adam Improves Muon: Adaptive Moment Estimation\\ with Orthogonalized Momentum}
\author{%
  \name Minxin Zhang \email minxinzhang@math.ucla.edu \\
  \name Yuxuan Liu \email yxliu@math.ucla.edu \\
  \name Hayden Schaeffer \email hayden@math.ucla.edu \\
  \addr Department of Mathematics \\
  University of California, Los Angeles \\
  Los Angeles, CA 90095, USA
}

\usepackage{xspace}
\newcommand{\NAMO}{NAMO\xspace}
\newcommand{\NAMOD}{NAMO\mbox{-}D\xspace}

\newcommand{\qed}{\hfill$\blacksquare$}

\begin{document}
\newcommand{\Hess}{\nabla^2}  
\renewcommand{\Re}{\mathbb{R}} 
\newcommand{\dotP}[2]{\left\langle #1, #2 \right\rangle}
\newcommand{\tu}{\tilde{u}}
\newcommand{\tf}{\tilde{f}}
\newcommand{\abs}[1]{\left|#1\right|}
\newcommand{\norm}[1]{\left\|#1\right\|}
\newcommand{\normF}[1]{\norm{#1}_{\scriptscriptstyle F}}
\newcommand{\Set}[1]{\left\{#1\right\}}
\newcommand{\sgn}{\text{sgn}}
\newcommand{\trace}[1]{\operatorname{Tr}\left(#1\right)}
\renewcommand{\vec}[1]{\operatorname{vec}\left(#1\right)}
\newcommand{\argmin}{\operatorname*{argmin}}
\newcommand{\spa}[1]{\operatorname*{span}\left\{#1\right\}}
\newcommand{\bigO}{\mathcal O}
\renewcommand{\L}{\mathcal L}
\newcommand{\N}{\mathcal N}
\newcommand{\X}{\mathcal X}
\newcommand{\Z}{\mathcal Z}
\newcommand{\K}{\mathcal K}
\newcommand{\bv}{\mathbf v}
\newcommand{\bd}{\mathbf d}
\newcommand{\Zd}{{\scriptscriptstyle \mathcal Z}}
\newcommand{\floor}[1]{\left\lfloor #1 \right\rfloor}
\newcommand{\ceil}[1]{\left\lceil #1 \right\rceil}
\newcommand{\pr}[1]{\mathbb P\left( #1 \right)}
\newcommand{\expect}{\mathbb E}
\newcommand{\Var}[1]{\textrm{Var} \left[#1\right]}
\newcommand{\Cov}[1]{\textrm{Cov} \left[#1\right]}
\newcommand{\xstar}{x^\ast}
\newcommand{\Grad}{\nabla\!}
\newcommand{\comp}[1]{\left[#1 \right]}
\newcommand{\rank}[1]{\textrm{rank} \left(#1\right)}

\newcommand{\GD}{\textrm{\tiny GD}}
\newcommand{\OGD}{\textrm{\tiny OGD}}
\newcommand{\orth}[1]{\operatorname{Orth}\left(#1\right)}
\maketitle

\begin{abstract} 
Efficient stochastic optimization typically integrates an update direction that performs well in the deterministic regime with a mechanism adapting to stochastic perturbations. While Adam uses adaptive moment estimates to promote stability, Muon utilizes the weight layers' matrix structure via orthogonalized momentum, showing superior performance in large language model training. We propose a new optimizer and a diagonal extension, \NAMO and \NAMOD, providing the first principled integration of orthogonalized momentum with norm-based Adam-type noise adaptation. \NAMO scales orthogonalized momentum using a single adaptive stepsize, preserving orthogonality while improving upon Muon at negligible additional cost. \NAMOD instead right-multiplies orthogonalized momentum by a diagonal matrix with clamped entries. 
This design enables neuron-wise noise adaptation and aligns with the common near block-diagonal Hessian structure. Under standard assumptions, we establish optimal convergence rates for both algorithms in the deterministic setting and show that, in the stochastic setting, their convergence guarantees adapt to the noise level of stochastic gradients. Experiments on pretraining GPT-2 models demonstrate improved performance of both \NAMO and \NAMOD compared to the AdamW and Muon baselines, with \NAMOD achieving further gains over \NAMO via an additional clamping hyperparameter that balances the competing goals of maintaining a well-conditioned update direction and leveraging fine-grained noise adaptation. 
\end{abstract}

\begingroup
\makeatletter
\renewcommand\@makefntext[1]{%
  \parindent 1em
  \noindent #1%
}
\footnotetext{~~~~The code is available at: \url{https://github.com/minxin-zhg/namo}.}
\makeatother
\endgroup

\section{Introduction}
Stochastic optimization is central to modern large-scale learning, where algorithms must update iterates using only noisy estimates of first-order information. A key challenge is to balance two goals: selecting an update direction that is effective in the noise-free regime, and incorporating mechanisms that adapt to stochastic perturbations. From this viewpoint, an efficient stochastic optimizer can be understood as combining two ingredients: (i) a principled direction-selection rule that performs well when gradients are exact, and (ii) an adaptive stepsize mechanism that stabilizes the iterates by attenuating updates under gradient uncertainty. This perspective offers a useful lens for interpreting existing methods and for guiding the development of new algorithms with strong convergence behavior. In this paper, we design a new optimization algorithm that couple the structural advantages of an orthogonalized update direction with adaptive moment estimation to account for gradient noise. We provide theoretical convergence guarantees and demonstrate strong empirical performance in training large language models (LLMs).

Adaptive methods such as Adam \citep{kingma2014adam} and its decoupled weight-decay variant AdamW \citep{loshchilov2017decoupled} have long been the de facto optimizers for large-scale training. Their coordinate-wise adaptive stepsizes promote training stability and reduce sensitivity to hyperparameter choices. Given momentum coefficients $\beta_1,\beta_2\in[0,1)$, Adam iteratively updates a biased first-moment estimate of the stochastic gradient:
\[
m_t = \beta_1 m_{t-1} + (1-\beta_1) g_t,
\]
and a biased second raw-moment estimate:
\[
v_t = \beta_2 v_{t-1} + (1-\beta_2) g_t^2,
\]
where $g_t$ denotes the stochastic gradient at iteration $t$, and the square is applied elementwise. With the standard bias corrections $\hat m_t :=m_t/(1-\beta_1^t)$ and $\hat v_t := v_t/(1-\beta_2^t)$, the parameters are updated via:
\[
\theta_t = \theta_{t-1} - \eta \,\frac{\hat m_t}{\sqrt{\hat v_t}+\epsilon},
\]
where $\epsilon>0$ is a small fixed scalar that avoids division by zero, and $\eta>0$ is the learning rate. Throughout the paper, we refer to the hyperparameter $\eta$
as the \emph{learning rate}, and to its combination with an adaptive scaling as the effective \emph{stepsize}.
The ratio of the first- and second-moment estimates, $\hat m_t/\sqrt{\hat v_t}$, is often interpreted as a {signal-to-noise ratio} (SNR). When the update direction is dominated by noise, or as the iterates approach a stationary point, the resulting effective stepsize decreases, which is desirable for stable convergence.
When $\beta_1=\beta_2,$ \citep{orvieto2025search} interprets Adam as performing online estimation of the mean and variance of the stochastic gradients. 
Moreover, Adam's update rule can be viewed as combining a sign descent direction with a variance adaptation component. 
Theoretical analysis and experiments in \citep{balles2018dissecting} show that the sign-descent component can lead to Adam’s 
adverse generalization behavior \citep{wilson2017marginal}, indicating that coupling Adam-type noise adaptation
with a different descent direction may yield improved generalization relative to Adam.

While Adam and other standard variants of stochastic gradient descent (SGD) treat trainable parameters of neural networks as flattened vectors, 
the Muon optimizer \citep{jordan2024muon} exploits their matrix structure by updating weight matrices with orthogonalized momentum.
Given a matrix $M\in\Re^{m\times n}$, if $M=U\Sigma V^T$ is its reduced singular value decomposition (SVD), then its orthogonalization is given by:
\[
\orth{M} := UV^T.
\] The orthogonalization of $M$, $\orth{M}$, is also called the orthogonal factor in the polar decomposition and is the nearest orthogonal matrix to $M$ in the Frobenius norm \citep{higham2008functions}. Muon updates matrix-structured parameters $\Theta\in\Re^{m\times n}$ by: \[
\Theta_t = \Theta_{t-1}-\eta O_t,
\]
where $\eta>0$ is the prescribed learning rate, and $O_t\approx \orth{M_t}$ denotes an approximate orthogonalization of the momentum
matrix $M_t$ at iteration $t$, computed via Newton-Schultz iterations \citep{modula-docs}.
Growing evidence demonstrates that Muon achieves superior empirical performance compared to Adam in LLM training 
\citep{jordan2024modded, team2025kimi, liu2025muon}. 
Recent studies show that orthogonalized updates can accelerate convergence, facilitate more reliable hyperparameter transfer across model sizes \citep{boreiko2025towards, shah2025practical}, and learn more effectively from heavy-tailed data \citep{wang2025muon}.
In the deterministic setting without momentum, Muon's orthogonalized update direction
can be interpreted as the steepest descent direction under the spectral norm, 
which is shown to be effective for training in deep learning \citep{davis2025spectral}.
In stochastic settings, however, matrix orthogonalization is an unbounded operation \citep{higham1986computing} that may amplify the impact of noise in the 
original momentum matrix, leading to training instability \citep{he2025root} 
and increased sensitivity to hyperparameter choices \citep{crawshaw2025exploration}.
This suggests that pairing Muon’s orthogonalized updates with an explicit noise-adaptation mechanism could improve training robustness 
and further boost performance.

In this work, we develop a theoretically principled integration of Adam-type variance adaptation with an orthogonalized update direction.
Prior work shows that orthogonalization decouples the direction and magnitude of a matrix update through a norm-duality characterization 
\citep[Proposition~5]{bernstein2024old}, suggesting that a norm-based rescaling of the orthogonalized update is a natural design choice.
We pair norm-based moment estimation with orthogonalized momentum and propose a new optimizer, \NAMO (\textbf{N}orm-Based \textbf{A}daptive \textbf{M}oment Estimation with \textbf{O}rthogonalized Momentum), together with a diagonal extension, \NAMOD.
\NAMO scales the orthogonalized momentum with a norm-based adaptive scalar, preserving the orthogonality of the update direction while
adapting the effective stepsize to the noise level. \NAMOD scales the orthogonalized momentum with a right-multiplied diagonal matrix, allowing an 
individual adaptive stepsize for each neuron. We establish theoretical convergence rates for both algorithms, and demonstrate that they outperform the AdamW and Muon 
baselines in GPT-2 pretraining. 

\subsection{Related work}
Unlike Adam’s coordinate-wise adaptive stepsizes, our methods introduce structured stepsize adaptation for orthogonalized updates: \NAMO scales the orthogonalized momentum using a single adaptive stepsize, whereas \NAMOD employs a column-wise adaptive stepsize for the orthogonalized momentum. Related simplifications of Adam have also been studied. In \citep{chezhegov2024clipping}, a clipped norm-based Adam-type stepsize is applied to the momentum, and convergence is analyzed under the assumptions of bounded gradients and heavy-tailed noise. Motivated by the near block-diagonal Hessian structure of neural networks, Adam-mini \citep{zhang2024adam} partitions parameters into blocks and assigns a single adaptive learning rate per block, matching Adam’s performance while reducing memory cost. 

Several adaptive variants of Muon have been proposed as well. AdaMuon \citep{si2025adamuon} and NorMuon \citep{li2025normuon} combine Muon’s orthogonalized momentum with Adam-type scaling variants, though without theoretical convergence guarantees. AdaGO \citep{zhang2025adagrad} scales the orthogonalized momentum with an adaptively decaying stepsize and achieves optimal theoretical convergence rates, although its performance in LLM training has yet to be investigated. Layerwise adaptive learning rates for Muon are proposed in \citep{hao2025noise} based on gradient-variance estimation that requires evaluating a nuclear norm, thereby increasing the per-iteration cost. Tuning-robust variants of Muon are explored in \citep{crawshaw2025exploration}. PRISM \citep{yang2026prism} augments Muon with a moment-based adaptive preconditioner, but incurs higher additional computational cost than our proposed algorithms, 
and does not provide theoretical convergence guarantees.
DeVA \citep{song2026decoupling} decouples variance adaptation from scale-invariant sign descent, yielding an
Adam-style scaling for Muon's orthogonalized momentum that require high computational and memory overhead for maintaining Kronecker preconditioners and periodic eigendecompositions.


\subsection{Contributions and Organization}
We propose a new optimization algorithm and a diagonal extension for problems with matrix-structured parameters, \NAMO and \NAMOD, 
which provide the first theoretically principled integration of an orthogonalized update direction with noise adaptation based on Adam-type moment estimation.
\NAMO scales the orthogonalized momentum using a single adaptive stepsize, incurring only a negligible $\bigO(mn)$ 
additional computational cost and no additional memory overhead, thus providing a useful improvement for Muon’s performance.
On the other hand, \NAMOD scales the orthogonalized momentum by right-multiplying with a diagonal matrix $D_t$, thereby assigning a neuron-wise adaptive stepsize determined by the column norms of the stochastic gradient and momentum matrices. 
This column-wise scaling enables finer-grained noise adaptation and aligns with the near block-diagonal Hessian structure commonly 
observed in neural networks \citep{dong2025towards, an2025asgo}, while no longer strictly preserving the orthogonalized update direction.
The diagonal entries of $D_t$ admit an upper bound and are clamped toward the average, ensuring that
\NAMOD's update direction remains well-conditioned.
Under standard smoothness and unbiased bounded-variance noise assumptions, we establish optimal convergence rates for both algorithms in the deterministic setting and 
show that, in the stochastic setting, their convergence guarantees adapt to the noise level of the stochastic gradients and attain the optimal rate when the batch size is sufficiently large.
Experiments on pretraining GPT-2 models demonstrate that both \NAMO and \NAMOD outperform AdamW and Muon baselines. 
\NAMOD achieves further gains over \NAMO through an additional clamping hyperparameter $c$, which balances the competing goals of maintaining a 
well-conditioned update direction and leveraging finer-grained noise adaptation.

The rest of the paper is organized as follows. Section~\ref{sec:alg} describes the proposed algorithms, 
and Section~\ref{sec:analysis} provides the convergence analysis, with detailed proofs deferred to Appendices~\ref{appendix:lemma}--\ref{pf:dsm}.
Section~\ref{sec:exp} presents experiments on GPT-2 pretraining, and Section~\ref{sec:conclude} concludes the paper.

\section{New Optimization Algorithms: \NAMO and \NAMOD}\label{sec:alg}
In this section, we describe the proposed algorithm \NAMO and its diagonal extension \NAMOD.
We consider optimization problems with  matrix-structured parameters of the form: \[
\min_{\Theta\in \Re^{m\times n}} \L(\Theta),
\] where $\L:\Re^{m\times n} \rightarrow \Re$ is the (nonconvex) loss function.
We impose the following standard smoothness of loss function and bounded-variance noise assumptions.
\begin{assum}\label{assum:func}
The gradient of $\L(\Theta)$ is Lipschitz continuous, i.e., for arbitrary $\Theta, \Theta'\in\Re^{m\times n}$:
\begin{equation}\label{eq:lipG}
\norm{\Grad\L(\Theta)-\Grad\L(\Theta')}_*\le L\norm{\Theta-\Theta'}_2,
\end{equation}
for some constant $L>0,$ where $\norm{\cdot}_*$ and $\norm{\cdot}_2$ denote the nuclear norm and the spectral norm respectively.
\end{assum}
Note that gradient of the loss, $\Grad\L(\Theta)$, is in $\Re^{m\times n}$ and recall that the nuclear norm is the sum of the singular values while the spectral norm is the maximum singular value. 
\begin{assum}\label{assum:noise}
At each iteration $t,$ the stochastic gradient $G_t$ is an unbiased estimate of the true gradient, i.e.,
$\expect[G_t] = \Grad\L(\Theta_{t-1}),$
with a uniformly bounded variance:
\[\expect\left[\normF{G_t-\Grad\L(\Theta_{t-1})}^2\right]\le \frac{\sigma^2}{b},\]
where $b\ge 1$ is the batch size and $\normF{\cdot}$ denotes the Frobenius norm.
\end{assum}
Assumption~\ref{assum:func} is equivalent to a more commonly seen smooth assumption:
\begin{equation}\label{eq:lipGF}
\normF{\Grad\L(\Theta)-\Grad\L(\Theta')}\le L'\normF{\Theta-\Theta'}
\end{equation}
for a different Lipschitz constant $L'>0$. 
Since orthogonalized gradient descent can be interpreted as steepest descent under the spectral norm \citep{shen2025convergence}, 
we state the smoothness assumption in the equivalent form~\eqref{eq:lipG}.

Details of \NAMO and \NAMOD are summarized in Algorithms~\ref{alg:namo} and \ref{alg:namod}, respectively. 
For the theoretical analysis in Section~\ref{sec:analysis}, we assume exact orthogonalization in both algorithms, 
as is standard in existing works \citep{shen2025convergence, sato2025analysis, li2025note, chen2025muon, kovalev2025understanding, pethick2025training}.
In practice, exact orthogonalization can be expensive to compute, 
and for the experiments in Section~\ref{sec:exp} we use Newton--Schulz iterations to obtain an approximate orthogonalization, as in the Muon optimizer.

\NAMO maintains a biased second raw-moment estimate of the squared Frobenius norm:  \[
v_t = \mu_2 v_{t-1}+(1-\mu_2)\normF{G_t}^2,
\] 
where $G_t$ is the stochastic gradient matrix at iteration $t$, and the momentum matrix $M_t$ is a biased first moment estimate of the stochastic gradient.
Applying bias correction yields $\hat M_t := M_t/(1-\mu_1^t)$ and $\hat v_t := v_t/(1-\mu_2^t)$.
The parameters $\Theta\in \Re^{m\times n}$ are updated by: \begin{equation}\label{eq:alphat0}
\Theta_{t} = \Theta_{t-1}-
\eta\alpha_t O_{t},
\end{equation}
where $\eta>0$ is a prescribed learning rate, and the orthogonalized momentum $O_t$ is adaptively scaled by a scalar: \[
\alpha_t:=\frac{ (1-\mu_2^t)^{\frac{1}{2}}}{1-\mu_1^t}\frac{\norm{M_t}}{\sqrt{v_{t}}+\epsilon}= \frac{\norm{\hat M_t}}{\sqrt{\hat v_t}+\epsilon_t},
\] 
with $\epsilon_t:=\epsilon/\sqrt{1-\mu_2^t}$ for a small $\epsilon>0$. When stochastic gradients are noisy, or when the iterates approach a stationary point,
$\alpha_t$ is small, promoting stable convergence. 

\begin{algorithm}[htp]
\caption{\NAMO: \small{\textbf{N}orm-Based \textbf{A}daptive \textbf{M}oment Estimation with \textbf{O}rthogonalized Momentum}}\label{alg:namo}
\begin{algorithmic}[1]
\Require Learning rate $\eta$, momentum $\mu_1,\mu_2\in [0,1)$, batch size $b$, $\epsilon>0$
\State Initialize $\Theta_0\in\Re^{m\times n}$, $M_{0} = 0$, $v_0=0$
\For{$t = 1, 2, \dots,T$}
    \State Sample a minibatch of size $b$ and compute stochastic gradient $G_{t} = \nabla \mathcal{L}_{t}(\Theta_{t-1})$
    \State $M_{t} \gets \mu_1 M_{t-1} + (1-\mu_1)G_t$
        \State $v_{t} \gets \mu_2 v_{t-1}+(1-\mu_2)\normF{G_t}^2$
    \State $O_{t} \gets \orth{M_{t}}$
    \State $\alpha_t \gets \frac{\sqrt{1-\mu_2^t}}{1-\mu_1^t}\frac{\normF{M_t}}{\sqrt{v_{t}}+\epsilon}$
    \State Update parameters $\Theta_{t} \gets \Theta_{t-1} \;-\; 
    \eta\alpha_t O_{t} $ 
\EndFor
\State \Return $\Theta_{T}$
\end{algorithmic}
\end{algorithm}

While \NAMO scales the orthogonalized momentum using a single adaptive stepsize, \NAMOD applies a column-wise scaling based on the column norms of the momentum matrix. This design assigns an individual adaptive stepsize to each neuron and is consistent with the near block-diagonal Hessian structure commonly observed in neural networks \citep{dong2025towards, an2025asgo}.
For each $j=1,2,\cdots,n$, \NAMOD maintains a biased second raw-moment estimate of the squared norm of the $j$-th colum of the stochastic gradient:
\[
\comp{\bv_t}_j = \mu_2 \comp{\bv_t}_{j-1}+(1-\mu_2)\norm{\comp{G_t}_{:j}}^2.
\] To write this in an equivalent vector form, define the operator $\N_c:\Re^{m\times n}\to\Re^n$ by: \[
\comp{\N_c(M)}_j := \norm{M_{:j}}, \quad j=1,2,\cdots,n,
\] where $\norm{\cdot}$ denotes the Euclidean norm. Then the iterates can be written as: \[
\bv_{t} = \mu_2 \bv_{t-1}+(1-\mu_2)\N_c(G_t)\odot \N_c(G_t).
\]
Applying bias correction yields $\hat M_t := M_t/(1-\mu_1^t)$ and $\hat \bv_t := \bv_t/(1-\mu_2^t)$. Now let: \[
\bd_t =\N_c(\hat M_t)\oslash\left(\sqrt{\hat \bv_t}+\epsilon_t\right),
\] where $\epsilon_t:=\epsilon/\sqrt{1-\mu_2^t}$ for a small $\epsilon>0$, and $\oslash$ denotes entrywise division. By Lemma~\ref{lem:snr}, 
the entries of $\bd_t$ are bounded above by:
\[
\norm{\bd_t}_\infty = \max_j \comp{\bd_t}_j\le\sqrt{\frac{1-\mu_1}{1-\mu_2}}.
\]
Compared with \NAMO, \NAMOD's column-wise scaling enables finer-grained noise adaptation, but it no
longer strictly preserves orthogonality of the update direction. Let \(
\bar d_t:=\norm{\bd_t}_1/n
\) be the average of the entries of $\bd_t$.
To ensure that the scaled direction remains well-conditioned, we clamp the adaptive stepsizes toward this average via: \[
\tilde \bd_t: = \min\left\{\max\left\{\bd_t, c \bar d_t\mathbf{1}\right\}, \bar d_t/c\right\},
\] for a prescribed constant $c\in (0,1],$ where the maximum and minimum operations are applied entrywise, and \(\mathbf{1}\in\Re^n\) denotes the vector of all-ones. 
Let $D_t:=\diag\left(\tilde \bd_t\right)$. \NAMOD updates the  parameters by: \[
\Theta_t = \Theta_{t-1}-\eta O_t D_t.
\]
This update rule combines column-wise, noise-adaptive scaling with a simple clamping safeguard, 
tempering updates when gradients are noisy or the iterates are near a stationary point, while keeping the scaled update direction well-conditioned.

\begin{algorithm}[htp]
\caption{\NAMOD: Diagonal exension of \NAMO}\label{alg:namod}
\begin{algorithmic}[1]
\Require Learning rate $\eta$, momentum $\mu_1,\mu_2\in [0,1)$, batch size $b$, $\epsilon>0,$ $c\in (0,1]$
\State Initialize $\Theta_0\in\Re^{m\times n}$, $M_{0} = 0$, $\bv_0=0$
\For{$t = 1, 2, \dots,T$}
    \State Sample a minibatch of size $b$ and compute stochastic gradient $G_{t} = \nabla \mathcal{L}_{t}(\Theta_{t-1})$
    \State $M_{t} \gets \mu_1 M_{t-1} + (1-\mu_1)G_t$
        \State $\bv_{t} \gets \mu_2 \bv_{t-1}+(1-\mu_2)\N_c(G_t)\odot \N_c(G_t)$
        \State $\bd_t\gets \frac{\sqrt{1-\mu_2^t}}{1-\mu_1^t}\N_c(M_t)\oslash\left(\sqrt{\bv_t}+\epsilon\right)$
        \State $\bar d_t\gets \norm{\bd_t}_1/n$
    \State $O_{t} \gets \orth{M_{t}}$
    \State $D_t \gets \diag\left(\min\left\{\max\left\{\bd_t,c \bar d_t \mathbf{1}\right\},\frac{1}{c}\bar d_t \mathbf{1}\right\}\right)$
    \State Update parameters $\Theta_{t} \gets \Theta_{t-1} \;-\; 
    \eta  O_{t} D_t$ 
\EndFor
\State \Return $\Theta_{T}$
\end{algorithmic}
\end{algorithm}

\section{Convergence Analysis}\label{sec:analysis}
In this section, we establish convergence rates for \NAMO (Algorithm~\ref{alg:namo}) and \NAMOD (Algorithm~\ref{alg:namod}) in both the deterministic and stochastic settings 
under the standard
Assumptions~\ref{assum:func}--\ref{assum:noise}. 
The analysis of \NAMO is presented in Section~\ref{sec:namo_analyze}, 
and the analysis of \NAMOD is presented in Section~\ref{sec:namod_analyze}.
We provide a proof sketch for each of the theorems below and include detailed proofs in Appendices~\ref{pf:dm}--\ref{pf:dsm}. Proofs of useful lemmata are in Appendix~\ref{appendix:lemma}. The proofs repeatedly use the bias-corrected
moment estimates and their convex-combination representations with coefficients:
\[
w_{1,t,\tau}:=\frac{1-\mu_1}{1-\mu_1^t}\mu_1^{t-\tau},
\qquad
w_{2,t,\tau}:=\frac{1-\mu_2}{1-\mu_2^t}\mu_2^{t-\tau},
\quad
\textrm{for }~ t\ge 1~ \textrm{and} ~\tau\in\{1,\dots,t\},
\]
which satisfy $\sum_{\tau=1}^t w_{1,t,\tau}=\sum_{\tau=1}^t w_{2,t,\tau}=1.$
This representation allows one to control the deviation between the bias-corrected momentum and the
current gradient via Assumption~\ref{assum:func}, as well as the magnitude of the bias-corrected second-moment estimate by
elementary geometric-series bounds.

\subsection{Analysis of \NAMO}\label{sec:namo_analyze}
We first analyze the convergence of \NAMO in the deterministic case where gradients are evaluated exactly.
The $\bigO(T^{-1/2})$ rate established in the theorem below is the best possible for deterministic first-order methods under Assumption~\ref{assum:func}; see \citep[Theorem~2]{carmon2020lower}.
\begin{thm}[\NAMO in the deterministic case]\label{thm:dm}
Suppose Assumptions~\ref{assum:func} holds. Let $\Set{\Theta_t}\subset\Re^{m\times n}$ be the sequence of iterates generated by
Algorithm~\ref{alg:namo} with full-batch gradients and $0\le \mu_1\le \mu_2<1$.
If choosing $\eta = \bigO\left(T^{-\frac{1}{2}}\right)$, $\epsilon=\bigO(T^{-\frac{1}{2}}),$ $\mu_1=\Theta(1)$ and $\mu_2=\Theta(1)$,
then for large $T>0$:
\begin{align*}
\frac{1}{T} \sum_{t=1}^T \normF{\Grad\L(\Theta_{t-1})}
\le  \bigO\left(T^{-\frac{1}{2}}\right).
\end{align*}
\end{thm}
\paragraph{Proof sketch.}
Let $\hat M_t:=M_t/(1-\mu_1^t)$ and $\hat v_t:=v_t/(1-\mu_2^t)$ denote the bias-corrected moments,
so that $\hat M_t=\sum_{\tau=1}^t w_{1,t,\tau}\nabla L(\Theta_{\tau-1})$ and
$\hat v_t=\sum_{\tau=1}^t w_{2,t,\tau}\|\nabla L(\Theta_{\tau-1})\|_F^2$.
Define the adaptive scaling $\alpha_t$ as in \eqref{eq:alphat0}. Applying Lemma~\ref{lem:snr} to the vectorization of the gradient
matrix sequence $\Set{\nabla L(\Theta_{\tau-1})}$ yields a uniform upper bound on $\alpha_t$ (Appendix~\ref{pf:dm}, Step 1).
Next, using the orthogonalized descent inequality from \citep[Lemma~B.1]{zhang2025adagrad} and Assumption~\ref{assum:func},
one obtains an average inequality involving:
\[
\frac{\|\hat M_t\|_F \, \|\nabla L(\Theta_{t-1})\|_*}{\sqrt{\hat v_t}+\epsilon/\sqrt{1-\mu_2^t}}
\quad\text{and}\quad
\|\nabla L(\Theta_{t-1})-\hat M_t\|_*
\]
 (Appendix~\ref{pf:dm}, Step 2). The deviation of the bias-corrected momentum from the gradient is then controlled using the convex-combination
representation of $\hat M_t$ together with Assumption~\ref{assum:func} and a telescoping bound on
$\|\Theta_{\tau-1}-\Theta_{t-1}\|_2$ (Appendix~\ref{pf:dm}, Step 3).
Similarly, $\sqrt{\hat v_t}$ is upper bounded by $\|\nabla L(\Theta_{t-1})\|_*$ plus a constant
obtained from geometric-series estimates (Appendix~\ref{pf:dm}, Step 4).
Combining these bounds and applying Lemma~\ref{lem:phi_eps} converts the bound on an averaged surrogate $\phi_\epsilon$ into
the stated bound on the average gradient norm (Appendix~\ref{pf:dm}, Steps 5--6).
The full proof is provided in Appendix~\ref{pf:dm}. \hfill \qed

\vspace{0.5em}
Next we analyze \NAMO in the stochastic case where stochastic gradients have bounded variance.
The best possible convergence rate for stochastic first-order methods under 
Assumptions~\ref{assum:func}--\ref{assum:noise} is $\bigO(T^{-1/4})$; see \citep[Theorem 3]{arjevani2023lower}.
The following result shows that \NAMO adapts to the noise level of the stochastic gradients and achieves the optimal convergence 
rate when the batch size is sufficiently large.
\begin{thm}[\NAMO in the stochastic case]\label{thm:sm}
Suppose Assumptions~\ref{assum:func}--\ref{assum:noise} hold. Let $\Set{\Theta_t}\subset\Re^{m\times n}$ be the sequence of iterates generated by
Algorithm~\ref{alg:namo}.
For large $T>0,$ if we set $\eta=\bigO(T^{-\frac{3}{4}})$, $1-\mu_1=\Theta(T^{-\frac{1}{2}})$, 
$1-\mu_2=\Theta(T^{-\frac{1}{2}}),$ $0\le \mu_1\le \mu_2<1$, and $\epsilon =\bigO(T^{-\frac{1}{2}})$, then: \[
\frac{1}{T}\sum_{t=1}^T\expect\left[\normF{\Grad\L(\Theta_{t-1})}\right]\le \bigO\left(T^{-\frac{1}{4}}+\sqrt{\sigma}b^{-\frac{1}{4}}T^{-\frac{1}{8}}\right),
\] where $b$ is the batch size.
\end{thm}
\paragraph{Proof sketch.}
The proof again works with the convex-combination representations of the bias-corrected moment estimates $\hat M_t$ and $\hat v_t$ with weight coefficients
$w_{1,t,\tau}$ and $w_{2,t,\tau}$ respectively.
Let $\expect_t[\cdot]$ be the expectation conditioned on the  iterate $\Theta_{t-1}$,  and define the deviation term 
$E_t:=\hat M_t-\nabla L(\Theta_{t-1})$.
Using \citep[Lemma~B.1]{zhang2025adagrad} and taking conditional expectations yields an average bound on
$\expect[\alpha_t\|\hat M_t\|_*]$ in terms of $\expect\left[\|E_t\|_*\right]$ and $\expect\left[\alpha_t^2\right]$ (Appendix~\ref{pf:sm}, Step 2).
The quantity $\expect\left[\|E_t\|_*\right]$ is controlled by decomposing $E_t$ into a weighted sum of noise terms
$G_\tau-\nabla L(\Theta_{\tau-1})$ and a drift term $\nabla L(\Theta_{\tau-1})-\nabla L(\Theta_{t-1})$. 
The noise term is bounded above via Assumption~\ref{assum:noise}, the drift term is bounded above using Assumption~\ref{assum:func}
along with a telescoping argument, and then Lemma~\ref{lem:mut} is applied to bound the averaging over $t$ (Appendix~\ref{pf:sm}, Step 3).
In parallel, Minkowski and Jensen's inequalities yield an upper bound on the expected denominator term
$\expect\left[\sqrt{\hat v_t}+\epsilon/\sqrt{1-\mu_2^t}\right]$ by $\expect\left[\norm{\nabla L(\Theta_{t-1})}\right]$ plus an explicit
additive term containing $\sigma/\sqrt{b}$, then Lemma~\ref{lem:mutsqrt} is applied to bound the averaging over $t$ (Appendix~\ref{pf:sm}, Step 4).
Finally, a Cauchy--Schwarz argument relates $\expect[\alpha_t\|\hat M_t\|_F]$ to
$\expect[\|\nabla L(\Theta_{t-1})\|_F]$, leading to a quadratic inequality in
$\expect\left[\normF{\nabla L(\Theta_{t-1})}\right]$ that can be solved explicitly (Appendix~\ref{pf:sm}, Step 5). Substituting the averaged
bounds and the parameter choices yields the stated rate. 
The full proof is provided in Appendix~\ref{pf:sm}. \qed

This result indicates that the convergence rate of \NAMO is adaptive to the noise level of the stochastic gradients, and that choosing the batch size as
$b = \Omega(\sigma^2 \sqrt{T})$ recovers the optimal $\bigO(T^{-1/4})$ rate.

\subsection{Analysis of \NAMOD}\label{sec:namod_analyze}
NAMO-D introduces a column-wise diagonal scaling $D_t$, defined in Algorithm~2, instead of the scalar
$\alpha_t$. We first establish its optimal $\bigO\left(T^{-1/2}\right)$ rate for the deterministic case where gradients are evaluated exactly
in the theorem below.
\begin{thm}[\NAMOD in the deterministic case]\label{thm:ddm}
Suppose Assumption~\ref{assum:func} holds. Let $\Set{\Theta_t}\subset\Re^{m\times n}$ be the sequence of iterates generated by
Algorithm~\ref{alg:namod} using full batch with $0\le \mu_1\le \mu_2<1$.
If we choose $\eta=\bigO(T^{-\frac{1}{2}})$, $1-\mu_1=\Theta(1)$, 
$1-\mu_2=\Theta(1),$ $0\le \mu_1\le \mu_2<1$, $\epsilon =\bigO(T^{-\frac{1}{2}}n^{-1})$,
and $c=\Theta(1),$  then:
\begin{align*}
\frac{1}{T} \sum_{t=1}^T \normF{\Grad\L(\Theta_{t-1})}
\le  \bigO\left(T^{-\frac{1}{2}}\right),
\end{align*}
for large $T>0$.
\end{thm}
\paragraph{Proof sketch.}
The proof works with the diagonal scaling matrix $D_t$ defined as in Section~\ref{sec:alg} and its extremal diagonal entries
$d_{t,\max}$ and $d_{t,\min}$.
Applying Lemma~\ref{lem:snr} column-wise yields a uniform upper bound on the entries of $D_t$ (Appendix~\ref{pf:ddm}, Step 1).
The clamping rule implies $d_{t,\min}\ge c^2\,d_{t,\max}$ and hence bounds the condition number of $D_t$.
Applying the orthogonalized descent inequality \citep[Lemma~B.1]{zhang2025adagrad} gives an average bound on
$d_{t,\max}\norm{\nabla L(\Theta_{t-1})}_*$ in terms of the deviation
$\norm{\nabla L(\Theta_{t-1})-\hat M_t}_*$ (Appendix~\ref{pf:ddm}, Step 2).
The deviation term and the term $\sqrt{\sum_{j=1}^n \hat v_t^j}$ are controlled using similar arguments as in the proof of Theorem~\ref{thm:dm}, together with geometric-series
bounds (Appendix~\ref{pf:ddm}, Steps 3--4). The result again follows by combining these bounds and applying Lemma~\ref{lem:phi_eps} (Appendix~\ref{pf:ddm}, Steps 5--6).
The full proof is provided in Appendix~\ref{pf:ddm}.\qed

\vspace{0.5em}
In the stochastic case, we show that the convergence of \NAMOD adapts to the noise level of
stochastic gradients, and achieves the optimal $\bigO\left(T^{-1/4}\right)$ rate when the batch size is sufficiently large.
\begin{thm}[\NAMOD in the stochastic case]\label{thm:dsm}
Suppose Assumptions~\ref{assum:func}--\ref{assum:noise} holds. Let $\Set{\Theta_t}\subset\Re^{m\times n}$ be the sequence of iterates generated by
Algorithm~\ref{alg:namod}.
For large $T>0,$ if we set $\eta=\bigO(T^{-\frac{3}{4}})$, $1-\mu_1=\Theta(T^{-\frac{1}{2}})$, 
$1-\mu_2=\Theta(T^{-\frac{1}{2}}),$ $0\le \mu_1\le \mu_2<1$, $\epsilon =\bigO(T^{-\frac{1}{2}})$, and 
$c =\Theta(1)$, then: \[
\frac{1}{T}\sum_{t=1}^T\expect\left[\normF{\Grad\L(\Theta_{t-1})}\right]\le \bigO\left(T^{-\frac{1}{4}}+\sqrt{\sigma}b^{-\frac{1}{4}}T^{-\frac{1}{8}}\right),
\] where $b$ is the batch size.
\end{thm}
\paragraph{Proof sketch.}
Define $\hat M_t$, $\hat v_t$, and $D_t$ as before, and let $\expect_t\left[\cdot\right]$ denotes the expectation conditioned on iterate $\Theta_{t-1}$, and
$E_t:=\hat M_t-\nabla L(\Theta_{t-1})$. The expected descent step uses Lemma~\ref{lem:od} to lower bound
$\langle \hat M_t, O_tD_t\rangle$ by $d_{t,\min}\norm{\hat M_t}_*$, and the clamping rule controls the ratio of 
$d_{t,\max}$ and $d_{t,\min}$ (Appendix~\ref{pf:dsm}, Step 2). The averaged deviation term
$\frac{1}{T}\sum_{t=1}^T\expect\left[\|E_t\|_*\right]$ is bounded by the same variance-plus-drift decomposition as in the proof of Theorem~\ref{thm:sm}, 
using Assumption~\ref{assum:noise} to control the noise term and Assumption~\ref{assum:func} to control the drift term via a
telescoping argument; then averaging uses Lemma~\ref{lem:mut} (Appendix~\ref{pf:dsm}, Step 3).
For the denominator, Minkowski and Jensen's inequalities yield an upper bound on
$\expect[\sqrt{\sum_{j=1}^n \hat v_t^j}]$ by
$\expect\left[\norm{\nabla L(\Theta_{t-1})}\right]$ plus an explicit additive term containing $\sigma/\sqrt{b}$; then the averaging over $t$ is bounded using
Lemma~\ref{lem:mutsqrt} (Appendix~\ref{pf:dsm}, Step 4).
Finally, a Cauchy--Schwarz argument relate 
$\expect\left[d_{t,\max}\normF{\hat M_t}\right]$ to the expected gradient norm, yielding a quadratic inequality in
$\expect\left[\norm{\nabla L(\Theta_{t-1})}\right]$ that can be solved explicitly (Appendix~\ref{pf:dsm}, Step 5). Combining these bounds with the specified
hyperparameter choices gives the stated convergence rate (Appendix~\ref{pf:dsm}, Step 6). 
The full proof is provided in Appendix~\ref{pf:dsm}.\qed

This result indicates that the convergence rate of \NAMOD, like that of \NAMO, is adaptive to the noise level of the stochastic gradients, and that the optimal 
$\bigO(T^{-1/4})$ rate is recovered when choosing the batch size as $b=\Omega(\sigma^2 \sqrt{T}),$ 


\vspace{1em}
The results of this section demonstrate that both \NAMO and \NAMOD can achieve optimal convergence rates in terms of the order on $T$. This occurs in both the deterministic and stochastic settings under Assumptions~\ref{assum:func}-\ref{assum:noise}. In the deterministic setting, the $\bigO(T^{-1/2})$ bound matches the known lower complexity bound for smooth nonconvex optimization with first-order methods, thus demonstrating that utilizing orthogonalized update and adaptive scaling does not incur any deterioration in the convergence or complexity. In the stochastic regime, the bounds decompose naturally into an optimization term of order $\bigO(T^{-1/4})$ and an explicit variance-dependent term scaling as $\bigO\left(\sqrt{\sigma}\,b^{-1/4}T^{-1/8}\right)$, and thus quantifying the precise affect of gradient noise and batch size. Specifically, when $b=\Omega(\sigma^2 T^{1/2})$, the variance term becomes asymptotically dominated and the optimal $\bigO(T^{-1/4})$ rate is recovered. The analysis further shows that the orthogonalized descent inequality, combined with bias-corrected moment estimates and controlled adaptive scaling (scalar or diagonal), yields a unified proof strategy that is robust to both drift and stochastic perturbations. Altogether, these guarantees provide a rigorous complexity characterization of the proposed methods and confirm that their matrix-aware adaptive scaling preserves optimal first-order performance while maintaining stability under noise.

\section{Experiments}\label{sec:exp}
\subsection{Experimental setup}
\par\noindent\textbf{Baselines.}
We compare our proposed algorithms, \NAMO and \NAMOD, against two popular baselines: AdamW \citep{loshchilov2017decoupled} and Muon \citep{jordan2024muon}. 
Since Muon, \NAMO, and \NAMOD are designed specifically for matrix parameters, 
we use AdamW to optimize all remaining scalar and vector parameters across all models.
For brevity, we refer to these hybrid optimizers simply as Muon, \NAMO, and \NAMOD.

\paragraph{Model Architectures.} We evaluate all optimizers on GPT-2 pretraining experiments, which are based on the nanoGPT \citep{karpathy2022nanogpt}
implementation of the GPT-2 architecture \citep{radford2019language}. Two model sizes are considered: small (124M parameters), and
medium (355M parameters).
All experiments are conducted on $4\times$ NVIDIA H100 GPUs.

\paragraph{Dataset.}
All experiments are conducted on the OpenWebText dataset \citep{gokaslan2019openwebtext}, 
which contains approximately 9 billion training tokens and 4.4 million validation tokens.

\paragraph{Training details.}
We use standard hyperparameter settings for the baselines. For AdamW, we set the momentum coefficients to $\beta_1 = 0.9$ and $\beta_2 = 0.95$. 
For Muon, the momentum coefficient is set to $\beta=0.95$. For \NAMO and \NAMOD, the momentum coefficients are
set to $\mu_1=0.95$ and $\mu_2=0.99$. The weight decay coefficient is set to $\lambda=0.01$ for all optimizers. The Muon optimizer applies decoupled weight decay as AdamW: \[
\Theta_{t} = \Theta_{t-1}-\eta \left(O_t+\lambda \Theta_{t-1}  \right),
\]
where $\eta$ is the prescribed learning rate.
For \NAMO and \NAMOD, as part of the effective stepsize,  the adaptive scaling is applied to the decoupled weight decay as well. Specifically, \NAMO updates by:
\[
\Theta_{t} = \Theta_{t-1}-\eta \alpha_t\left(O_t+\lambda \Theta_{t-1}  \right),
\]
where $\alpha_t$ is given in Algorithm~\ref{alg:namo}, and \NAMOD updates by:
\[
\Theta_{t} = \Theta_{t-1}-\eta \left(O_t+\lambda \Theta_{t-1}  \right)D_t,
\]
where $D_t$ is given in Algorithm~\ref{alg:namod}.
For GPT-2 (124M) model pretraining, we train with context length $L=1024$ and an effective batch size of 480 sequences (491,520 tokens) per optimizer update for all optimizers (micro-batch size $B=60$ and gradient accumulation $K=8$). 
For GPT-2 (355M) model, we also train with context length $L=1024$ and an effective batch size of 480 sequences (491,520 tokens) per optimizer update for all optimizers (micro-batch size $B=40$ and gradient accumulation $K=12$). 
For both models and all optimziers, we use the following learning rate scheduler: $2000$-step linear warm-up followed by constant learning rate. Also, for both models, we sweep for optimal learning rate $\eta$ for each optimizer through a grid search. 
For \NAMOD, we fix the clamping hyperparameter $c=0.1$ for GPT-2 (124M) experiments,
and sweep for optimal $c$ for GPT-2 (355M) experiments.

\subsection{Empirical Results on GPT-2}
For pretraining the GPT-2 (124M) model, we perform a grid search for each optimizer to find the optimal learning rate (LR) that achieves the lowest validation loss after 
$10$K steps. The sweeping results are presented in Figure~\ref{fig:gpt2s_sweep}, which shows that \NAMO and \NAMOD achieve lower training and validation losses 
across a wider range of learning-rate choices, demonstrating both accelerated convergence and improved tuning robustness comparing to Muon and AdamW.

The optimal learning rate selected from the 10K-step sweep for each optimizer is listed in Table~\ref{tab:hyper}. 
While \NAMO and \NAMOD exhibit similar training and validation losses in the 10K-step learning-rate sweep, 
Figure~\ref{fig:gpt2s} shows that when training is extended to 50K steps using the selected optimal learning rate, 
\NAMOD attains lower training and validation losses than \NAMO, demonstrating advantages of its finer-grained neuron-wise adaptive scaling.
The training and validation losses at termination for each optimizer are reported in Table~\ref{tab:losses}.

We further conduct experiments on pretraining the GPT-2 (355M) model. For each optimizer, we sweep from five learning rates that are around the chosen optimal learning
rate for GPT-2 (124M) model listed in Table~\ref{tab:hyper}. Specifically, for Muon and AdamW, we sweep for the optimal LR from the following set: $$\Set{0.0006,0.0009,0.0013,0.0018,0.0025};$$
for \NAMO and for \NAMOD, we sweep for the optimal LR from the following set: $$\Set{0.005,0.007, 0.009,0.012, 0.015}.$$
The optimal LR is the one that achieves the lowest 
validation loss after $10$K steps.
Additionally, for \NAMOD, we observe that its performance on the GPT-2 (355M) model can vary for different choices of the clamping hyperparameter $c.$ Therefore, we also
sweep for an optimal $c$ from the following set: $$\Set{0.12, 0.40, 0.75, 0.90}.$$
The optimal hyperparameters for each optimizer for the GPT-2 (355M) model are listed in Table~\ref{tab:hyper}, and the training and validation losses at termination for each optimizer are reported in Table~\ref{tab:losses}, together with those for the GPT-2 (124M) model.
Figure~\ref{fig:gpt2m} show that \NAMO and \NAMOD outperforms both Muon and AdamW, with \NAMOD providing further gains over \NAMO through the additional hyperparameter $c.$
This observation is consistent with the algorithmic design of the two proposed algorithms. \NAMO augments Muon with adaptivity while preserving the orthogonality of the
update direction. On the other hand, \NAMOD enables neuron-wise adaptivity while no longer strictly preserving the orthogonalized direction, and a suitable choice 
of the clamping parameter $c$ balances the two competing goals of maintaining the structural advantages of a well-conditioned update direction and
benefiting from finer-grained noise adaptation.

\begin{table}[htp]
\caption{\textbf{Optimal hyperparameters for all optimizers and pretraining of two model sizes.} Grid search is used to determine the optimal choices. }
\vspace{0.5em}
\centering
\begin{tabular}{lcc}
\toprule
Optimizer & GPT-2 (124M) Hyperparameters & GPT-2 (355M) Hyperparameters\\
\midrule
AdamW & $\eta=0.0013$ & $\eta=0.0009$\\
Muon & $\eta=0.0013$ & $\eta=0.0009$\\
\NAMO & $\eta=0.012$ & $\eta=0.007$\\
\NAMOD & $\eta=0.009$, $c=0.1$ & $\eta=0.009$, $c=0.9$\\
\bottomrule
\end{tabular}
\label{tab:hyper}
\end{table}

\begin{figure}[thp]
  \centering
  \subfloat[Training loss\label{fig:gpt2s_train_lr_sweep}]{
    \includegraphics[width=0.48\linewidth]{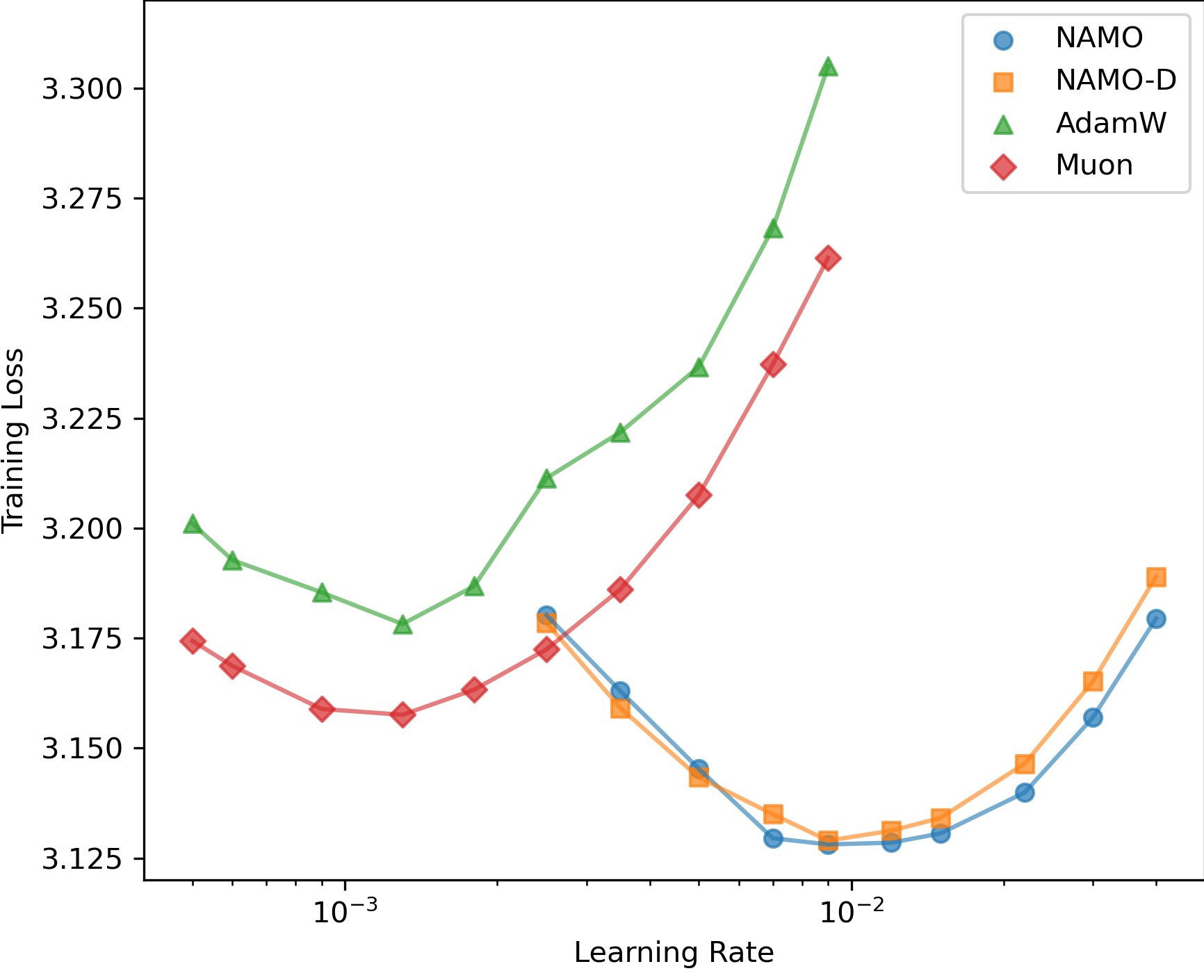}
  }
  \hfill
  \subfloat[Validation Loss\label{fig:gpt2s_val_lr_sweep}]{
    \includegraphics[width=0.48\linewidth]{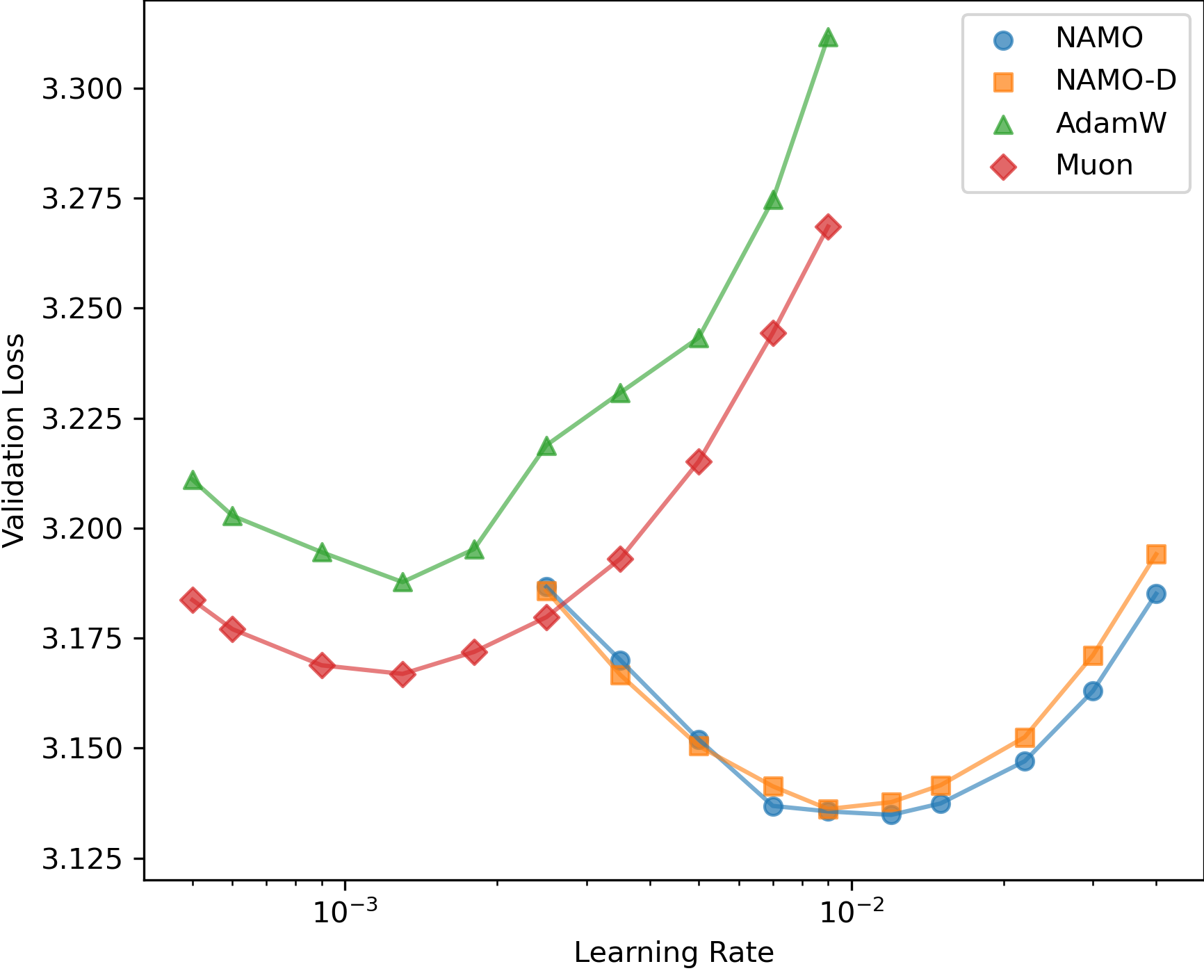}
  }
  \caption{\textbf{Hyperparameter sweeping results for GPT-2 (124M).} The training and validation losses at step 10K are reported, where the x-axis is the learning rate.}
\label{fig:gpt2s_sweep}
\end{figure}

\begin{figure}[thp]
  \centering
  \subfloat[Training loss\label{fig:gpt2s_train_50K}]{
    \includegraphics[width=0.48\linewidth]{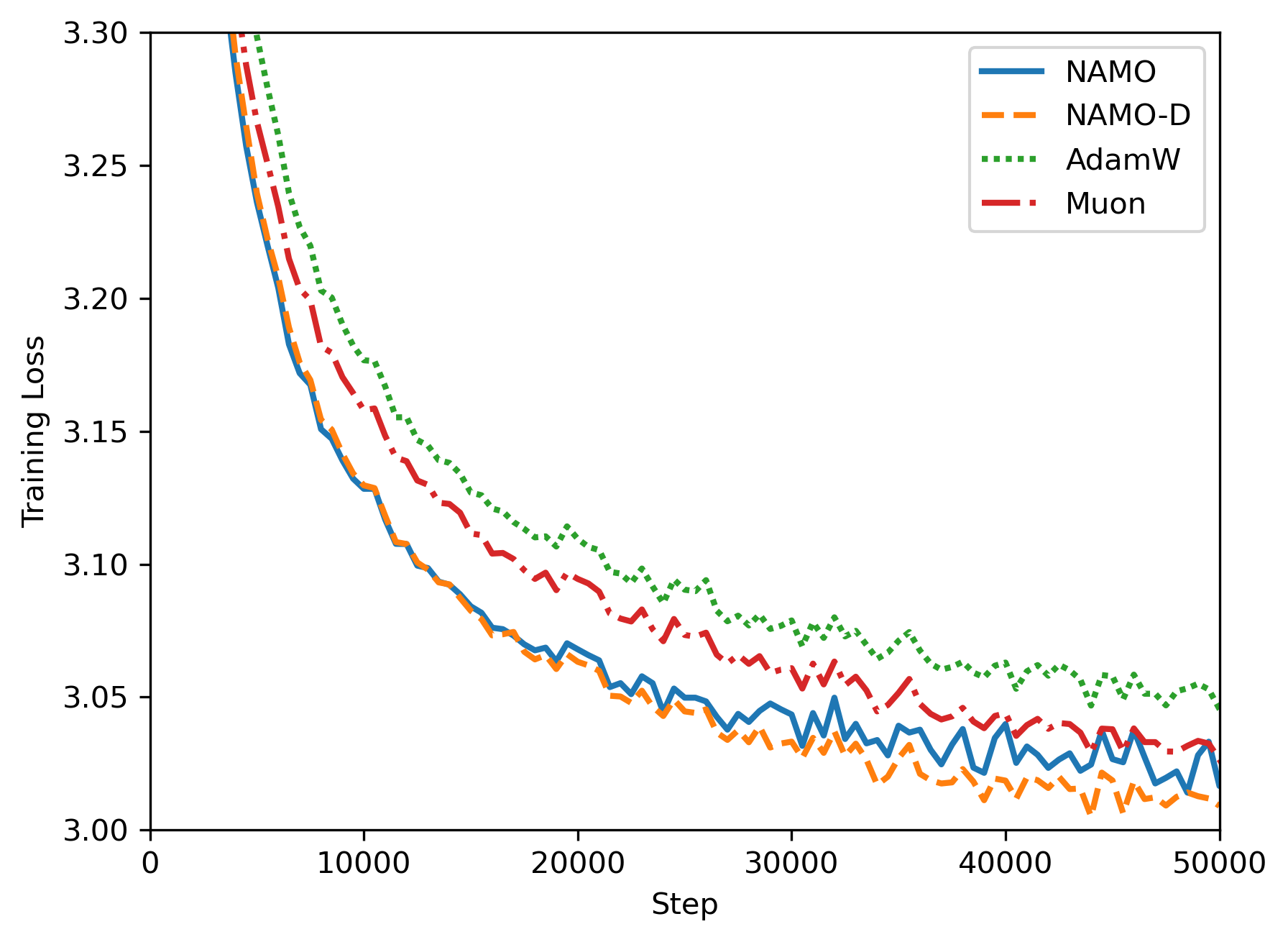}
  }
  \hfill
  \subfloat[Validation Loss\label{fig:gpt2s_val_50K}]{
    \includegraphics[width=0.48\linewidth]{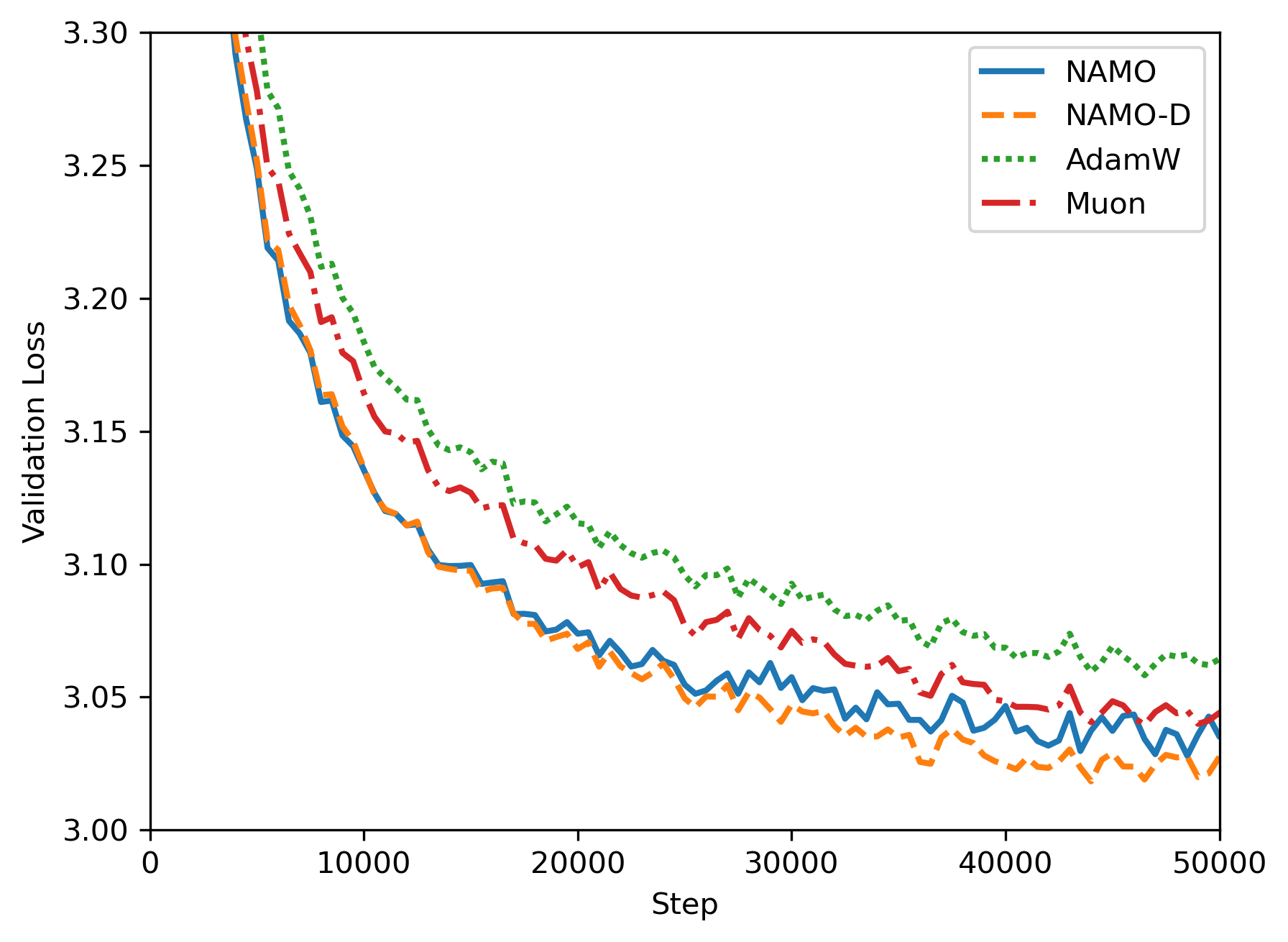}
  }
  \caption{\textbf{Pretraining GPT-2 (124M) for 50K steps.} The optimal LR from sweeping for 10K steps is used. }
\label{fig:gpt2s}
\end{figure}

\begin{table}[thp]
\centering
\caption{\textbf{Final training and validation losses for GPT-2 (124M) and GPT-2 (355M).} GPT-2 (124M) is trained for 50K steps, and GPT-2 (355M)
is trained for 10K steps.}
\label{tab:final-loss}
\vspace{0.5em}
\begin{tabular}{l cc cc}
\toprule
\multirow{2}{*}{Optimizer} & \multicolumn{2}{c}{GPT-2 (124M)} & \multicolumn{2}{c}{GPT-2 (355M)} \\
\cmidrule(lr){2-3}\cmidrule(lr){4-5}
 & Training Loss & Validation Loss & Training Loss & Validation Loss \\
\midrule
AdamW  & {3.0456} & {3.0643} & {2.9760} & {2.9914} \\
Muon   & {3.0265} & {3.0435} & {2.9524} & {2.9684} \\
\NAMO  & {2.9272} & {3.0351} & {2.9359} & {2.9516} \\
\NAMOD & {2.9167} & {3.0246} & {2.9351} & {2.9507} \\
\bottomrule
\end{tabular}
\label{tab:losses}
\end{table}

\begin{figure}[htp]
  \centering
  \subfloat[Training loss\label{fig:gpt2m_train_10K}]{
    \includegraphics[width=0.48\linewidth]{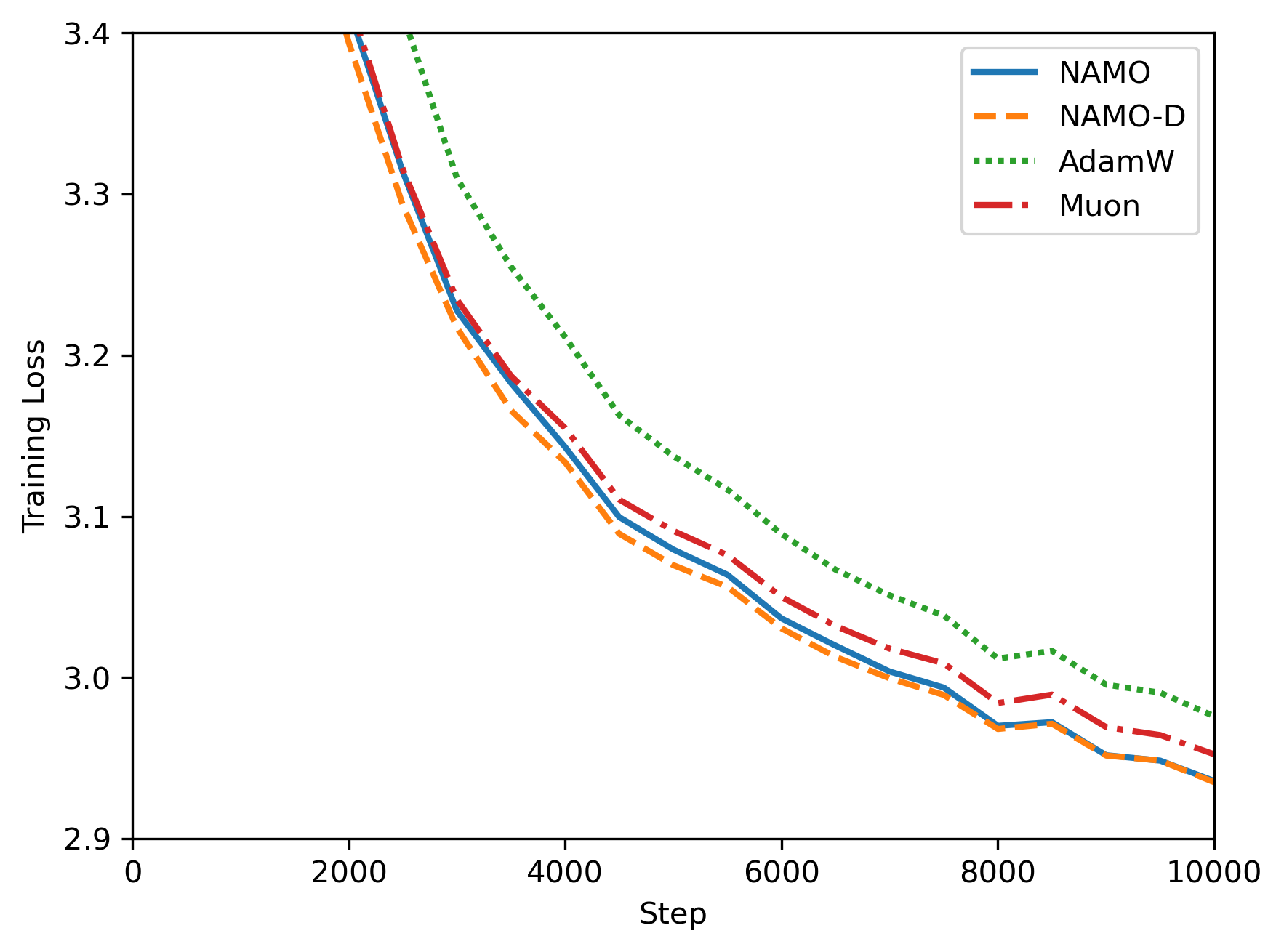}
  }
  \hfill
  \subfloat[Validation Loss\label{fig:gpt2m_val_10K}]{
    \includegraphics[width=0.48\linewidth]{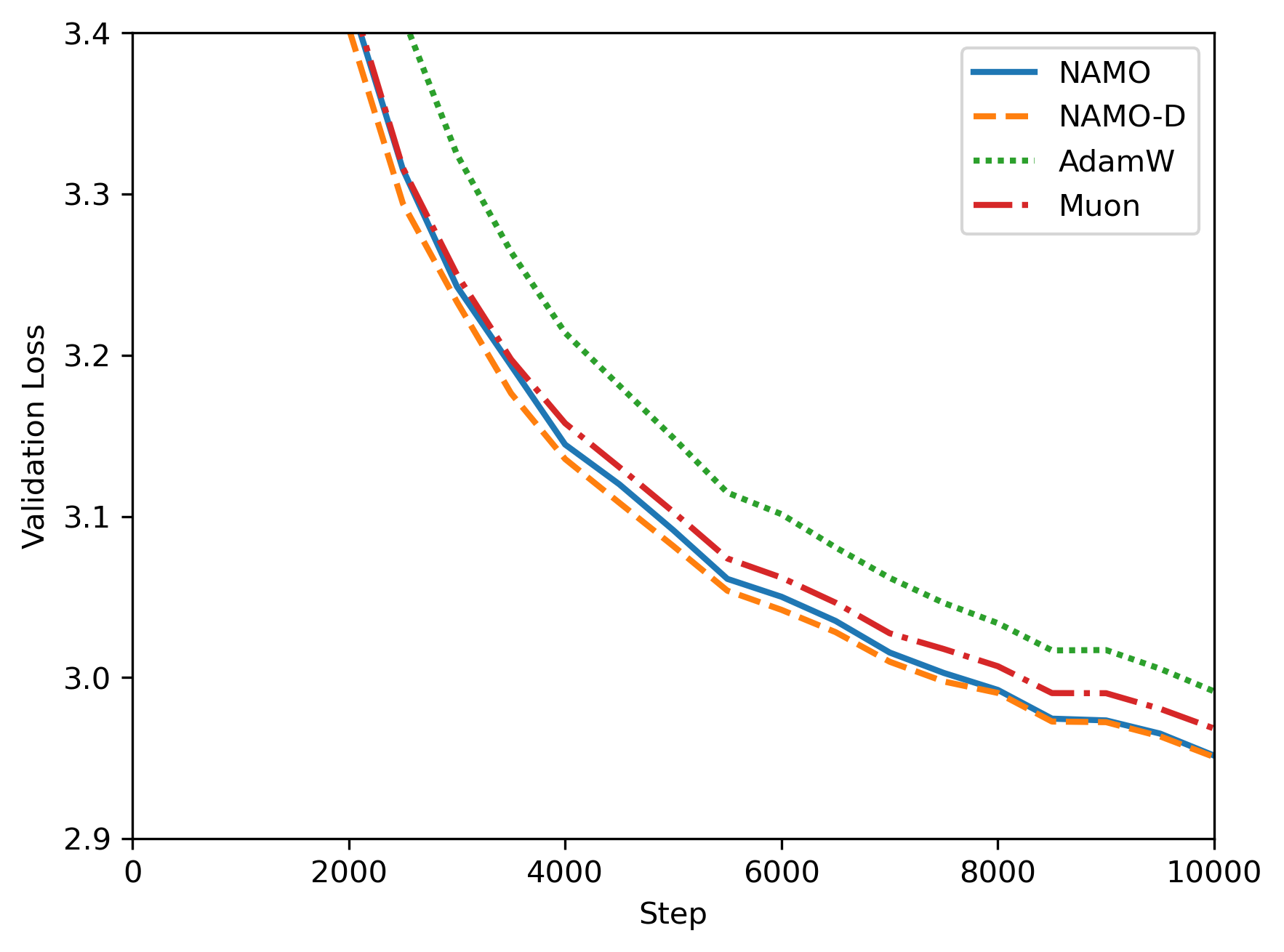}
  }
  \caption{\textbf{Pretraining GPT-2 (355M) for 10K steps.}
  The optimal LR (and optimal $c$ for \NAMOD) from sweeping for 10K steps are used.
  }
\label{fig:gpt2m}
\end{figure}

\section{Conclusions and Future Work}\label{sec:conclude}
In this work, we propose a new optimizer and a diagonal extension, \NAMO and \NAMOD, which provide the first theoretically principled integration of an orthogonalized update direction with norm-based 
adaptive moment estimation for noise adaptation. \NAMO rescales Muon’s orthogonalized momentum by a single adaptive scalar, thereby 
preserving the orthogonality of the update direction while yielding performance improvements over Muon at negligible additional computational cost. 
\NAMOD instead right-multiplies the orthogonalized momentum by a diagonal matrix, enabling finer-grained, neuron-wise noise adaptation for further 
improving performance, albeit without strictly preserving orthogonality of the update direction. Under standard smoothness and unbiased bounded-variance noise assumptions, we establish optimal convergence rates for both algorithms in the deterministic setting and show that, in the stochastic setting, their convergence guarantees 
adapt to the noise level of the stochastic gradients and attain the optimal rate when the batch size is sufficiently large. Experiments on pretraining GPT-2 (124M) and
GPT-2 (355M) models demonstrate the improved performance of both \NAMO and \NAMOD compared to the AdamW and Muon baselines. 
\NAMOD provides further gains over \NAMO through an additional clamping hyperparameter, which can be tuned to balance two competing goals: maintaining the structural advantages 
of a well-conditioned update direction and leveraging finer-grained noise adaptation. Future work includes evaluating \NAMO and \NAMOD on larger LLMs, 
developing tuning-light variants of \NAMOD, and further investigating theoretical and practical advantages of noise-adaptive scaling for orthogonalized updates.

\section*{Acknowledgement}
\paragraph{Funding.} This work was supported in part by NSF DMS 2502561.

\bibliography{ref}

\appendix
\section{Useful Lemmata}\label{appendix:lemma}
This section contains proofs of lemmata used in the analysis of \NAMO and \NAMOD.
The following lemma is used in the proofs of Theorems~\ref{thm:dm}--\ref{thm:dsm} for the derivation of a uniform upper bound of the adaptive stepsizes.
\begin{lem}\label{lem:snr}
Assume $\mu_1,\mu_2\in (0,1]$ with $\mu_1\le \mu_2$, $t>0,$ and $g_1,g_2,\cdots,g_t\in\Re^d.$ Define $m_t$ by $m_0=0$ and: \[
m_\tau = \mu_1 m_{\tau-1}+(1-\mu_1)g_\tau, ~\tau=1,2,\cdots,t-1.
\] Define $v_t$ by $v_0=0$ and: \[
v_\tau = \mu_2 m_{\tau-1}+(1-\mu_2)\norm{g_\tau}^2, ~\tau=1,2,\cdots,t-1.
\] Let $\hat m_t:=m_t/(1-\mu_1^t)$ and $\hat v_t:=v_t/(1-\mu_2^t).$
Then: \[\frac{\norm{\hat m_t}}{\sqrt{\hat v_t}}\le \sqrt{\frac{1-\mu_1}{1-\mu_2}}.\]
\end{lem}
\begin{proof}
Write $w_{1,t,\tau}:=\frac{1-\mu_1}{1-\mu_1^t}\mu_1^{t-\tau}$ and $w_{2,t,\tau}:=\frac{1-\mu_2}{1-\mu_2^t}\mu_2^{t-\tau}$.
Then: \[
\hat m_t=\frac{1-\mu_1}{1-\mu_1^t}\sum_{\tau=1}^t\mu_1^{t-\tau}g_\tau=\sum_{\tau=1}^t w_{1,t,\tau}g_\tau,
\] and \[
\hat v_t=\frac{1-\mu_2}{1-\mu_2^t}\sum_{\tau=1}^t\mu_2^{t-\tau}\norm{g_\tau}^2=\sum_{\tau=1}^tw_{2,t,\tau}\norm{g_\tau}^2.
\]
Since $\sum_{\tau=1}^tw_{1,t,\tau}=\sum_{\tau=1}^tw_{2,t,\tau}=1.$ Under the assumption that $\mu_1\le\mu_2,$ \[
w_{1,t,\tau} = \frac{1-\mu_1}{1-\mu_2}\frac{1-\mu_2^t}{1-\mu_1^t}\left(\frac{\mu_1}{\mu_2}\right)^{t-\tau}w_{2,t,\tau}\le \frac{1-\mu_1}{1-\mu_2}w_{2,t,\tau}.
\] By Cauchy-Schwarz inequality, \begin{align*}
\norm{\hat m_t}^2 = & \norm{\sum_{\tau=1}^t w_{1,t,\tau}g_\tau}^2 \\
\le & \left(\sum_{\tau=1}^t w_{1,t,\tau}\right)\left(\sum_{\tau=1}^t w_{1,t,\tau}\norm{g_\tau}^2\right)\\
\le & \frac{1-\mu_1}{1-\mu_2}\sum_{\tau=1}^t w_{2,t,\tau}\norm{g_\tau}^2\\
=& \frac{1-\mu_1}{1-\mu_2}\hat v_t.
\end{align*} Hence, \begin{equation*}
\alpha_t\le \frac{\norm{\hat m_t}}{\sqrt{\hat v_t}}\le \sqrt{\frac{1-\mu_1}{1-\mu_2}}.
\end{equation*}
\end{proof}

The lemma below is used in the analysis of \NAMO and \NAMOD in the deterministic setting, i.e. in the proofs of Theorem~\ref{thm:dm} and Theorem~\ref{thm:ddm}.
\begin{lem}\label{lem:phi_eps}
For $\epsilon>0$ and $x\ge 0,$ define $\phi_\epsilon(x):=\frac{x^2}{x+\epsilon}.$ Then: \[
x\le  \phi_\epsilon(x)
+\sqrt{\epsilon\phi_\epsilon(x)}.
\]
\end{lem}
\begin{proof}
By the definition of $\phi_\epsilon(x)$, \[
x^2-\phi_\epsilon(x)x-\epsilon\phi_\epsilon(x)\le 0.
\] Solving for $x$ gives: \begin{equation*}
x\le \frac{\phi_\epsilon(x)+\sqrt{\phi_\epsilon(x)^2+4\epsilon\phi_\epsilon(x)}}{2}\le \phi_\epsilon(x)
+\sqrt{\epsilon\phi_\epsilon(x)}.
\end{equation*} 
\end{proof}



The following two lemmas are used in the proofs of Theorem~\ref{thm:sm} and Theorem~\ref{thm:dsm}.
\begin{lem}\label{lem:mut}
For $\mu\in (0,1)$ and $T>0,$ \[
\sum_{t=1}^T \frac{1}{1-\mu^t}\le T+\frac{\mu}{1-\mu}-\frac{1}{\ln\mu}\ln\left(\frac{1-\mu^T}{1-\mu}\right).
\]
\end{lem}
\begin{proof}
For $x\ge 1,$ define: \[
f(x):=\frac{\mu^x}{1-\mu^x}.
\] Since: \[
f'(x) = \frac{\mu^x\ln\mu}{(1-\mu^x)^2}<0,
\] the integral test gives: \[
\sum_{t=1}^T \frac{\mu^t}{1-\mu^t} = \sum_{t=1}^T f(t)\le f(1)+\int_{1}^T f(x) dx.
\] To compute the integral, write $y:=\mu^x, $ then $dy=\mu^x\ln\mu dx=y\ln\mu dx,$ and \[
\int_{1}^T \frac{\mu^x}{1-\mu^x} dx =\int_{\mu}^{\mu^T} \frac{y}{1-y}\frac{dy}{y\ln\mu}=-\frac{1}{\ln\mu}\int_{\mu}^{\mu^T}d\ln (1-y)
=-\frac{1}{\ln\mu}\ln\left(\frac{1-\mu^T}{1-\mu}\right).
\] Hence, \[
\sum_{t=1}^T \frac{\mu^t}{1-\mu^t}\le \frac{\mu}{1-\mu}-\frac{1}{\ln\mu}\ln\left(\frac{1-\mu^T}{1-\mu}\right).
\] It then follows that: \[
\sum_{t=1}^T \frac{1}{1-\mu^t} = \sum_{t=1}^T \left(1+\frac{\mu^t}{1-\mu^t}\right)\le T+\frac{\mu}{1-\mu}-\frac{1}{\ln\mu}\ln\left(\frac{1-\mu^T}{1-\mu}\right).
\]
\end{proof}

\begin{lem}\label{lem:mutsqrt}
For $\mu\in (0,1)$ and $T>0,$ \[
\sum_{t=1}^T \frac{1}{\sqrt{1-\mu^t}}\le T-\frac{2\ln(1+\sqrt{1-\mu^T})}{\ln\mu}.
\]
\end{lem}
\begin{proof}
For $x\ge 1,$ define: \[
f(x):=\frac{1}{\sqrt{1-\mu^x}}.
\] Since: \[
f'(x) = \frac{\mu^x\ln\mu}{2(1-\mu^x)^{3/2}}<0,
\] f is strictly decreasing on $[1,\infty)$. Hence, \[
\sum_{t=1}^T \frac{1}{\sqrt{1-\mu^t}}\le \int_0^T f(x) dx.
\] Let $y:=\sqrt{1-\mu^x},$ then $dy = \frac{-\mu^x\ln\mu}{2\sqrt{1-\mu^x}} dx = \frac{(y^2-1)\ln\mu}{2y}dx.$ It follows that: \[
\sum_{t=1}^T \frac{1}{\sqrt{1-\mu^t}}\le \int_0^T \frac{1}{\sqrt{1-\mu^x}} dx =\frac{-2}{\ln\mu}\int_0^{\sqrt{1-\mu^T}} \frac{1}{1-y^2} dy =\frac{1}{-\ln\mu}
\ln\left(\frac{1+\sqrt{1-\mu^T}}{1-\sqrt{1-\mu^T}}\right).
\] Since:\[
\frac{1}{-\ln\mu}\ln\left(\frac{1+\sqrt{1-\mu^T}}{1-\sqrt{1-\mu^T}}\right) 
=\frac{1}{-\ln\mu} \ln\left(\frac{\left(1+\sqrt{1-\mu^T}\right)^2}{1-\left(1-\mu^T\right)}\right) 
=T-\frac{2\ln(1+\sqrt{1-\mu^T})}{\ln\mu},
\] Hence, \[
\sum_{t=1}^T \frac{1}{\sqrt{1-\mu^t}}\le T-\frac{2\ln(1+\sqrt{1-\mu^T})}{\ln\mu}.
\]

\end{proof}

The lemma below is used for the analysis for \NAMOD, i.e., in the proofs of Theorem~\ref{thm:ddm} and Theorem~\ref{thm:dsm}.
\begin{lem}\label{lem:od}
Let $M\in\Re^{m\times n}$ have reduced singular value decomposition (SVD) $M=U\Sigma V^T$,
and let $O=UV^T.$ Let $D=\diag\left(d_1,\cdots,dn\right)$ with $d_i\ge 0$ for all $i,$ and write
$d_{\min}:=\min_i d_i.$ Then $$\dotP{M}{OD}\ge d_{\min}\norm{M}_*.$$
\end{lem}
\begin{proof}
Note that: \[
\dotP{M}{OD} = \trace{M^T OD} = \trace{(U\Sigma V^T)^TUV^TD}=\trace{V\Sigma V^T D}=\trace{\Sigma V^TDV}.
\] Since $D \succeq d_{\min} I_n$, \[
V^T(D-d_{\min}I_n )V\succeq 0. 
\] It follows that: \[
\dotP{M}{OD}=d_{\min}\trace{\Sigma}+\trace{\Sigma V^T(D-d_{\min}I_n )V}\ge d_{\min}\norm{M}_*.
\]
\end{proof}

\section{Proof of Theorem~\ref{thm:dm}}\label{pf:dm}
This section contains the detailed proof of Theorem~\ref{thm:dm} for the convergence of \NAMO in the deterministic setting.
\begin{proof}
\par\noindent\textbf{Step 1: A uniform upper bound on stepsize.} For each $t\ge 0$ and $\tau\le t,$  define: \[
w_{1,t,\tau}:=\frac{1-\mu_1}{1-\mu_1^t}\mu_1^{t-\tau} ~\textrm{ and }~ w_{2,t,\tau}:=\frac{1-\mu_2}{1-\mu_2^t}\mu_2^{t-\tau}.
\]
It satisfies that $\sum_{\tau=1}^tw_{1,t,\tau}=\sum_{\tau=1}^tw_{2,t,\tau}=1.$
Let $\hat M_t:=  M_t/(1-\mu_1^t)$ and $\hat v_t:=v_t/(1-\mu_2^t),$ then:
\[
\hat M_t=\frac{1-\mu_1}{1-\mu_1^t}\sum_{\tau=1}^t\mu_1^{t-\tau}\Grad\L(\Theta_{\tau-1})=\sum_{\tau=1}^t w_{1,t,\tau}\Grad\L(\Theta_{\tau-1}),
\] and \[
\hat v_t=\frac{1-\mu_2}{1-\mu_2^t}\sum_{\tau=1}^t\mu_2^{t-\tau}\norm{\Grad\L(\Theta_{\tau-1})}^2=\sum_{\tau=1}^tw_{2,t,\tau}\norm{\Grad\L(\Theta_{\tau-1})}^2,
\]
Define:
\[
\alpha_t:= \frac{\sqrt{1-\mu_2^t}}{1-\mu_1^t}\frac{\normF{M_t}}{\sqrt{v_t}+\epsilon}=\frac{\normF{\hat M_t}}{\sqrt{\hat v_t}+\epsilon/\sqrt{1-\mu_2^t}}.
\]
Since the Frobenius norm of a matrix coincides with the Euclidean norm of its vectorization, by Lemma~\ref{lem:snr},
\begin{equation}\label{eq:alphat}
\alpha_t\le \frac{\normF{\hat M_t}}{\sqrt{\hat v_t}}=\sqrt{\frac{1-\mu_1}{1-\mu_2}}.
\end{equation}

\par\noindent\textbf{Step 2: A descent inequality and averaging.}
By \citep[Lemma B.1]{zhang2025adagrad},
\begin{align*}
 & \L(\Theta_{t})-\L(\Theta_{t-1})\\
\le & -\dotP{\Grad\L(\Theta_{t-1})}{\eta\alpha_t O_t} +\frac{ L}{2}\eta^2\alpha_t^2\\
\le &-\eta\alpha_t\norm{\Grad\L(\Theta_{t-1})}_* +2\eta\alpha_t \norm{\Grad\L(\Theta_{t-1})-\hat M_t}_*+\frac{ L}{2}\eta^2\alpha_t^2\\
= & -\eta\frac{\norm{\hat M_t}\norm{\Grad\L(\Theta_{t-1})}_*}
{\sqrt{\hat v_t}+\epsilon/\sqrt{1-\mu_2}}+2\eta\alpha_t \norm{\Grad\L(\Theta_{t-1})-\hat M_t}_*+\frac{ L}{2}\eta^2\alpha_t^2.
\end{align*}
Rearranging the terms gives: \begin{align}\label{eq:phiTm}
\frac{1}{T}\sum_{t=1}^T \frac{\norm{\hat M_t}\norm{\Grad\L(\Theta_{t-1})}_*}
{\sqrt{\hat v_t}+\epsilon/\sqrt{1-\mu_2}}\le & \frac{\Delta}{\eta T}+\frac{2}{T}\sum_{t=1}^T\alpha_t\norm{\Grad\L(\Theta_{t-1})-\hat M_t}_*+\frac{\eta L}{2T}\sum_{t=1}^T\alpha_t^2\nonumber\\
\le &  \frac{\Delta}{\eta T}+\frac{2}{T}\sqrt{\frac{1-\mu_1}{1-\mu_2}}\sum_{t=1}^T\norm{\Grad\L(\Theta_{t-1})-\hat M_t}_*
+\frac{\eta L}{2}\left(\frac{1-\mu_1}{1-\mu_2}\right),
\end{align}
where $\Delta:=\L(\Theta_0)-\min_\Theta \L(\Theta).$

\par\noindent\textbf{Step 3: Bounding the distance between bias-corrected momentum and true gradient.}
The difference between the scaled momentum $\hat M_t$ and the gradient $\Grad\L(\Theta_{t-1})$ can be bounded by: \begin{align*}
\norm{\hat M_t-\Grad\L(\Theta_{t-1})}_*\le & \sum_{\tau=1}^tw_{1,t,\tau}\norm{\Grad\L(\Theta_{\tau-1})-\Grad\L(\Theta_{t-1})}_*\\
\le &  \sum_{\tau=1}^tw_{1,t,\tau}L\norm{\Theta_{\tau-1}-\Theta_{t-1}}_2\\
\le & \sum_{\tau=1}^tw_{1,t,\tau}L\norm{\sum_{s=\tau}^{t-1}\eta\alpha_sO_s}_2\\
\le &\sum_{\tau=1}^tw_{1,t,\tau}\eta L\left(\sum_{s=\tau}^{t-1}\alpha_s\right).
\end{align*} Then by \eqref{eq:alphat} and the definition of $w_{1,t,\tau}$, \begin{align}\label{eq:Mtilde}
\norm{\hat M_t-\Grad\L(\Theta_{t-1})}_*\le & \sum_{\tau=1}^tw_{1,t,\tau}\eta L\left(t-\tau\right)\sqrt{\frac{1-\mu_1}{1-\mu_2}}\nonumber\\
= &  \sum_{\tau=1}^t \eta L\left(t-\tau\right)\mu_1^{t-\tau}\frac{1-\mu_1}{1-\mu_1^t}\sqrt{\frac{1-\mu_1}{1-\mu_2}}\nonumber\\
\le & \eta L \frac{1-\mu_1}{1-\mu_1^{t}}\sqrt{\frac{1-\mu_1}{1-\mu_2}}\sum_{j=1}^\infty j\mu_1^j\nonumber\\
\le & \eta L\sqrt{\frac{1-\mu_1}{1-\mu_2}}\frac{\mu_1}{(1-\mu_1)^2}.
\end{align}

\par\noindent\textbf{Step 4: Bounding $\sqrt{\hat v_t}$.}
For $\hat v_t$, it satisfies that: \begin{align*}
\hat v_t = & \sum_{\tau=1}^t w_{2,t,\tau}\normF{\Grad\L(\Theta_{\tau-1})}^2\\
\le & \sum_{\tau=1}^t w_{2,t,\tau}\left(\norm{\Grad\L(\Theta_{t-1})}_*+L\norm{\Theta_{\tau-1}-\Theta_{t-1}}_2\right)^2\\
\le & \sum_{\tau=1}^t w_{2,t,\tau}\left(\norm{\Grad\L(\Theta_{t-1})}_*+\eta L\left(\sum_{s=\tau}^{t-1}\alpha_s\right)\right)^2\\
\end{align*}
Then by \eqref{eq:alphat} and the definition of $w_{2,t,\tau}$, \begin{align*}
\hat v_t \le &\sum_{\tau=1}^t w_{2,t,\tau}\left(\norm{\Grad\L(\Theta_{t-1})}_*+\eta L\left(t-\tau\right)\sqrt{\frac{1-\mu_1}{1-\mu_2}}\right)^2\\
= &\norm{\Grad\L(\Theta_{t-1})}_*^2+2\eta L\norm{\Grad\L(\Theta_{t-1})}_*\sqrt{\frac{1-\mu_1}{1-\mu_2}}\sum_{\tau=1}^t w_{2,t,\tau}(t-\tau)\\
& + \eta^2 L^2\left(\frac{1-\mu_1}{1-\mu_2}\right)\sum_{\tau=1}^t w_{2,t,\tau}(t-\tau)^2.
\end{align*}
By Cauchy-Schwarz inequality, it follows that: \begin{align*}
\hat v_t  \le & \norm{\Grad\L(\Theta_{t-1})}_*^2+2\eta L\norm{\Grad\L(\Theta_{t-1})}_*\sqrt{\frac{1-\mu_1}{1-\mu_2}}\sqrt{\sum_{\tau=1}^t w_{2,t,\tau}(t-\tau)^2}\\
& + \eta^2 L^2\left(\frac{1-\mu_1}{1-\mu_2}\right)\left(\sum_{\tau=1}^t w_{2,t,\tau}(t-\tau)^2\right)\\
\le & \left(\norm{\Grad\L(\Theta_{t-1})}_*+\eta L\sqrt{\frac{1-\mu_1}{1-\mu_2}}\sqrt{\sum_{\tau=1}^t w_{2,t,\tau}(t-\tau)^2}\right)^2\\
=& \left(\norm{\Grad\L(\Theta_{t-1})}_*+a_t\right)^2,
\end{align*} where \[
a_t:= \eta L\sqrt{\frac{1-\mu_1}{1-\mu_2}}\sqrt{\sum_{\tau=1}^t w_{2,t,\tau}(t-\tau)^2}.
\] Since \begin{align}\label{eq:csquared}
a_t^2 \le 
& \eta^2 L^2\left(\frac{1-\mu_1}{1-\mu_2}\right)\left(\frac{1-\mu_2}{1-\mu_2^t}\right)\left(\sum_{\tau=1}^t \mu_2^{t-\tau}(t-\tau)^2\right)\nonumber\\
\le & \eta^2 L^2\left(\frac{1-\mu_1}{1-\mu_2}\right)\left(\sum_{\tau=1}^\infty \mu_2^{\tau}\tau^2\right)\nonumber\\
\le & \eta^2 L^2\left(\frac{1-\mu_1}{1-\mu_2}\right)\left(\frac{\mu_2(1+\mu_2)}{(1-\mu_2)^3}\right):=a^2.
\end{align} Hence, \begin{equation}\label{eq:vtilde}
\sqrt{\hat v_t}\le \norm{\Grad\L(\Theta_{t-1})}_*+a.
\end{equation}

\par\noindent\textbf{Step 5: Relating $\alpha_t\|\nabla\L(\Theta_{t-1})\|_*$ to $\|\nabla\L(\Theta_{t-1})\|_*$.}
Combining \eqref{eq:Mtilde} and \eqref{eq:vtilde} gives: \begin{align*}
\alpha_t\norm{\Grad\L(\Theta_{t-1})}_*= & \frac{\normF{\hat M_t}\norm{\Grad\L(\Theta_{t-1}}_*}{\sqrt{\hat v_t}+\epsilon/\sqrt{1-\mu_2^t}}\\
\ge & \frac{\left(\normF{\Grad\L(\Theta_{t-1}}-\normF{\hat M_t-\Grad\L(\Theta_{t-1})}\right)\norm{\Grad\L(\Theta_{t-1}}_*}
{\norm{\Grad\L(\Theta_{t-1})}_*+\epsilon/\sqrt{1-\mu_2^t}+a}\\
\ge & \frac{\normF{\Grad\L(\Theta_{t-1}}\norm{\Grad\L(\Theta_{t-1}}_*}{\norm{\Grad\L(\Theta_{t-1})}_*+\epsilon/\sqrt{1-\mu_2^t}+a} 
-\normF{\hat M_t-\Grad\L(\Theta_{t-1})}\\
\ge & \frac{\normF{\Grad\L(\Theta_{t-1}}\norm{\Grad\L(\Theta_{t-1}}_*}{\norm{\Grad\L(\Theta_{t-1})}_*+\epsilon/(1-\mu_2^t)}\left(1-\frac{a}
{\norm{\Grad\L(\Theta_{t-1})}_*+\epsilon/(1-\mu_2^t)+a}\right)-\normF{\hat M_t-\Grad\L(\Theta_{t-1})}\\
\ge & \frac{\normF{\Grad\L(\Theta_{t-1}}\norm{\Grad\L(\Theta_{t-1}}_*}{\norm{\Grad\L(\Theta_{t-1})}_*+\epsilon/(1-\mu_2^t)}
-a-\normF{\hat M_t-\Grad\L(\Theta_{t-1})}\\
\ge & \frac{\normF{\Grad\L(\Theta_{t-1}}\norm{\Grad\L(\Theta_{t-1}}_*}{\norm{\Grad\L(\Theta_{t-1})}_*+\epsilon/(1-\mu_2^t)}-a-\eta L\sqrt{\frac{1-\mu_1}{1-\mu_2}}\frac{\mu_1}{(1-\mu_1)^2}\\
\ge & \frac{\norm{\Grad\L(\Theta_{t-1}}_*^2/\sqrt{r}}{\norm{\Grad\L(\Theta_{t-1})}_*+\tilde \epsilon}-a-\eta L\sqrt{\frac{1-\mu_1}{1-\mu_2}}\frac{\mu_1}{(1-\mu_1)^2}
\end{align*}
where $\tilde \epsilon := \epsilon/(1-\mu_2)$, $r:=\min\{m,n\}$
and the constant $a$ is given in \eqref{eq:csquared}.
Define $\phi_{\tilde \epsilon}(x):=\frac{x^2}{x+\tilde \epsilon}$ for $x\ge 0,$ it then follows from \eqref{eq:phiTm} and \eqref{eq:Mtilde} that: \begin{align*}
&\frac{1}{T}\sum_{t=1}^T \phi_{\tilde \epsilon}(\norm{\Grad\L(\Theta_{t-1}}_*)\\
\le & \sqrt{r}\left(\frac{1}{T}\sum_{t=1}^T 
\frac{\norm{\hat M_t}\norm{\Grad\L(\Theta_{t-1}}_*}{\sqrt{\hat v_t}+\epsilon/(1-\mu_2^t)}+a+\eta L\sqrt{\frac{1-\mu_1}{1-\mu_2}}\frac{\mu_1}{(1-\mu_1)^2}\right)\\
\le & \sqrt{r}\left(\frac{\Delta}{\eta T}+\frac{2}{T}\sqrt{\frac{1-\mu_1}{1-\mu_2}}\sum_{t=1}^T\norm{\Grad\L(\Theta_{t-1})-\hat M_t}_*
+\frac{\eta L}{2}\left(\frac{1-\mu_1}{1-\mu_2}\right)
+a+\eta L\sqrt{\frac{1-\mu_1}{1-\mu_2}}\frac{\mu_1}{(1-\mu_1)^2}\right)\\
\le &  \sqrt{r}\left(\frac{\Delta}{\eta T}+\frac{2}{T}\sqrt{\frac{1-\mu_1}{1-\mu_2}}\sum_{t=1}^T\norm{\Grad\L(\Theta_{t-1})-\hat M_t}_*
+\frac{\eta L}{2}\left(\frac{1-\mu_1}{1-\mu_2}\right)
+a+\eta L\sqrt{\frac{1-\mu_1}{1-\mu_2}}\frac{\mu_1}{(1-\mu_1)^2}\right)\\
\le & \sqrt{r}\left(\frac{\Delta}{\eta T}+\left(\frac{1-\mu_1}{1-\mu_2}\right)\frac{2\eta L\mu_1}{(1-\mu_1)^2}
+\frac{\eta L}{2}\left(\frac{1-\mu_1}{1-\mu_2}\right)
+\eta L\frac{\sqrt{\mu_2(1-\mu_1)(1+\mu_2)}}{(1-\mu_2)^2}+\sqrt{\frac{1-\mu_1}{1-\mu_2}}\frac{\eta L\mu_1}{(1-\mu_1)^2}\right)\\
=& \frac{\Delta \sqrt{r}}{\eta T}+\eta L C_{\mu}\sqrt{r},
\end{align*} where \[
C_\mu:= \left(\frac{1-\mu_1}{1-\mu_2}\right)\frac{2\mu_1}{(1-\mu_1)^2}
+\frac{1}{2}\left(\frac{1-\mu_1}{1-\mu_2}\right)
+\frac{\sqrt{\mu_2(1-\mu_1)(1+\mu_2)}}{(1-\mu_2)^2}+\sqrt{\frac{1-\mu_1}{1-\mu_2}}\frac{\mu_1}{(1-\mu_1)^2}
\] is a constant that depends on $\mu_1$ and $\mu_2$. 

\par\noindent\textbf{Step 6: Deriving the convergence rate.}
By Lemma~\ref{lem:phi_eps}, \begin{align*}
\frac{1}{T}\sum_{t=1}^T \norm{\Grad\L(\Theta_{t-1}}_*\le & \frac{1}{T}\sum_{t=1}^T \phi_{\tilde\epsilon}\left(\norm{\Grad\L(\Theta_{t-1}}_*\right)
+\frac{1}{T}\sum_{t=1}^T \sqrt{\tilde \epsilon\phi_{\tilde\epsilon}\left(\norm{\Grad\L(\Theta_{t-1}}_*\right)}\\
\le & \frac{1}{T}\sum_{t=1}^T \phi_{\tilde\epsilon}\left(\norm{\Grad\L(\Theta_{t-1}}_*\right)
+\frac{\tilde\epsilon^{\frac{1}{2}}}{\sqrt{T}}\sqrt{\sum_{t=1}^T \phi_{\tilde\epsilon}\left(\norm{\Grad\L(\Theta_{t-1}}_*\right)}\\
\le & \frac{\Delta\sqrt{r}}{\eta T}+\eta L C_{\mu}\sqrt{r}  + \tilde \epsilon^{\frac{1}{2}}\sqrt{\frac{\Delta\sqrt{r}}{\eta T}+\eta L C_{\mu}\sqrt{r}}.
\end{align*}

In particular, if choosing $\eta = \bigO\left(T^{-\frac{1}{2}}\right)$, $\mu_1=\Theta(1),$ $\mu_2=\Theta(1)$, and $\epsilon=\bigO(T^{-\frac{1}{2}}),$ then: \[
\frac{1}{T}\sum_{t=1}^T\normF{\Grad\L(\Theta_{t-1})}\le\frac{1}{T}\sum_{t=1}^T\norm{\Grad\L(\Theta_{t-1})}_*\le\bigO\left(T^{-\frac{1}{2}}\right)
\]
for large $T>0$. The proof is thus completed.
\end{proof}

\section{Proof of Theorem~\ref{thm:sm}}\label{pf:sm}
This section contains the detailed proof of Theorem~\ref{thm:sm} for the convergence of \NAMO in the stochastic setting.
\begin{proof}
\par\noindent\textbf{Step 1: A uniform upper bound on stepsize}
For each $t\ge 0$ and $\tau\le t,$  define: \[
w_{1,t,\tau}:=\frac{1-\mu_1}{1-\mu_1^t}\mu_1^{t-\tau} ~\textrm{ and }~ w_{2,t,\tau}:=\frac{1-\mu_2}{1-\mu_2^t}\mu_2^{t-\tau}.
\]
It satisfies that $\sum_{\tau=1}^tw_{1,t,\tau}=\sum_{\tau=1}^tw_{2,t,\tau}=1.$
Let $\hat M_t:=  M_t/(1-\mu_1^t)$ and $\hat v_t:=v_t/(1-\mu_2^t),$ then:
\[
\hat M_t=\frac{1-\mu_1}{1-\mu_1^t}\sum_{\tau=1}^t\mu_1^{t-\tau}G_\tau=\sum_{\tau=1}^t w_{1,t,\tau}G_\tau,
\] and \[
\hat v_t=\frac{1-\mu_2}{1-\mu_2^t}\sum_{\tau=1}^t\mu_2^{t-\tau}\normF{G_\tau}^2=\sum_{\tau=1}^tw_{2,t,\tau}\normF{G_\tau}^2,
\]
Define: 
\[
\alpha_t:= \frac{\sqrt{1-\mu_2^t}}{1-\mu_1^t}\frac{\normF{M_t}}{\sqrt{v_t}+\epsilon}=\frac{\normF{\hat M_t}}{\sqrt{\hat v_t}+\epsilon/\sqrt{1-\mu_2^t}}.
\]

Since the Frobenius norm of a matrix coincides with the Euclidean norm of its vectorization, by Lemma~\ref{lem:snr},
\begin{equation}\label{eq:alphats}
\alpha_t\le \frac{\normF{\hat M_t}}{\sqrt{\hat v_t}}=\sqrt{\frac{1-\mu_1}{1-\mu_2}}.
\end{equation}

\par\noindent\textbf{Step 2: Expected descent and an average bound on $\alpha_t\|\hat M_t\|_*$.}
Let $\expect_t[\cdot]:=\expect[\cdot|\Theta_{t-1}]$ denote the conditional expectation given the previous iterates
$\Theta_0,\cdots,\Theta_{t-1},$ and write $E_t:=\hat M_t-\Grad\L(\Theta_{t-1})$. Then by \citep[Lemma B.1]{zhang2025adagrad},
\begin{align*}
&\expect_t\left[\L(\Theta_{t})-\L(\Theta_{t-1})\right]\\
\le &\expect_t\left[-\dotP{\Grad\L(\Theta_{t-1})}
{ \eta\alpha_t O_t}\right]+  \frac{\eta^2L}{2}\expect_t\left[ \alpha_t^2\right]\\
= &\expect_t\left[-\dotP{\Grad\L(\Theta_{t-1})-\hat M_t}{ \eta\alpha_t O_t}\right] 
-\expect_t\left[\eta\alpha_t\norm{\hat M_t}_*\right]+
\frac{\eta^2 L}{2}\expect_t\left[ \alpha_t^2\right]\\
\le & \Biggl(\expect_t\left[ \eta\alpha_t\norm{\Grad\L(\Theta_{t-1})-\hat M_t}_*\right] 
-\expect_t\left[\eta\alpha_t\norm{\hat M_t}_*\right] \Biggr)
+\frac{\eta^2 L}{2}\expect_t\left[ \alpha_t^2\right]\\
\le &  -\expect_t\left[\eta\alpha_t\norm{\hat M_t}_*\right] 
+\expect_t\left[ \eta\alpha_t\norm{E_t}_*\right] 
+\frac{\eta^2 L}{2}\expect_t\left[ \alpha_t^2\right].
\end{align*}
Rearranging the terms gives: \[
\expect_t\left[\alpha_t\norm{\hat M_t}_*\right]\le \expect_t\left[\L(\Theta_{t-1})-\L(\Theta_{t})\right]+
\expect_t\left[\alpha_t\norm{E_t}_*\right] 
+\frac{\eta L}{2}\expect_t\left[ \alpha_t^2\right].
\]
Then by the law of total expectation and \eqref{eq:alphats},
\begin{align}\label{eq:totalM}
\frac{1}{T}\sum_{t=1}^T\expect\left[\alpha_t\norm{\hat M_t}_*\right]\nonumber
\le & \frac{\Delta}{\eta T}+ \frac{1}{T}\sum_{t=1}^T\expect\left[\alpha_t\norm{E_t}_*\right] 
+\frac{\eta L}{2T}\sum_{t=1}^T \expect\left[ \alpha_t^2\right]\\
\le & \frac{\Delta}{\eta T}+ \frac{1}{T}\sqrt{\frac{1-\mu_1}{1-\mu_2}}\sum_{t=1}^T\expect\left[\norm{E_t}_*\right] 
+\frac{\eta L (1-\mu_1)}{2 (1-\mu_2)}.
\end{align}

\par\noindent\textbf{Step 3: Bounding the distance between bias-corrected momentum and true gradient.}
For each $t$, it satisfies: \begin{align*}
E_t = & \hat M_t - \expect\left[\hat M_t\right] + \expect\left[\hat M_t\right] -\Grad\L(\Theta_{t-1})\\
= & \sum_{\tau=1}^t w_{1,t,\tau} \left(G_\tau-\Grad\L(\Theta_{\tau-1})\right)+\sum_{\tau=1}^t w_{1,t,\tau} 
\left(\Grad\L(\Theta_{\tau-1})-\Grad\L(\Theta_{t-1})\right)
\end{align*} Hence, by \eqref{eq:alphats}, \begin{align*}
\expect\left[\normF{E_t}^2\right] \le & \expect\left[\normF{\sum_{\tau=1}^t w_{1,t,\tau} \left(G_\tau-\Grad\L(\Theta_{\tau-1})\right)}^2\right] 
+\normF{\sum_{\tau=1}^t w_{1,t,\tau} \left(\Grad\L(\Theta_{\tau-1})-\Grad\L(\Theta_{t-1})\right)}^2\\
\le & \left(\frac{1-\mu_1}{1-\mu_1^t}\right)^2\left(\sum_{\tau=1}^t\mu_1^{2(t-\tau)}\right)\frac{\sigma^2}{b}
+\sum_{\tau=1}^t w_{1,t,\tau}\norm{\Grad\L(\Theta_{\tau-1})-\Grad\L(\Theta_{t-1})}_*^2\\
\le & \left(\frac{1-\mu_1}{1-\mu_1^t}\right)^2\left(\frac{1-\mu_1^{2t}}{1-\mu_1^2}\right)\frac{\sigma^2}{b}
+\left(\frac{1-\mu_1}{1-\mu_1^t}\right)\sum_{\tau=1}^t\mu_1^{t-\tau}L^2\norm{\Theta_{\tau-1}-\Theta_{t-1}}_2^2\\
\le & \left(\frac{1-\mu_1}{1-\mu_1^t}\right)\left(\frac{1+\mu_1^{t}}{1+\mu_1}\right)\frac{\sigma^2}{b}
+ \left(\frac{1-\mu_1}{1-\mu_1^t}\right)L^2\eta^2\sum_{\tau=1}^t\mu_1^{t-\tau}\left(\sum_{s=\tau}^{t-1}\alpha_s\right)^2\\
\le & \left(\frac{1-\mu_1}{1-\mu_1^t}\right)\left[\left(\frac{1+\mu_1^{t}}{1+\mu_1}\right)\frac{\sigma^2}{b}
+ \left(\frac{1-\mu_1}{1-\mu_2}\right)L^2\eta^2\sum_{\tau=1}^t\mu_1^{t-\tau}\left(t-\tau\right)^2\right]\\
\le & \left(\frac{1-\mu_1}{1-\mu_1^t}\right)\left[\left(\frac{1+\mu_1^{t}}{1+\mu_1}\right)\frac{\sigma^2}{b}
+ \left(\frac{1-\mu_1}{1-\mu_2}\right)L^2\eta^2\sum_{\tau=1}^\infty\mu_1^{\tau}\tau^2\right]\\
= & \left(\frac{1-\mu_1}{1-\mu_1^t}\right)\left[\left(\frac{1+\mu_1^{t}}{1+\mu_1}\right)\frac{\sigma^2}{b}
+ \left(\frac{\mu_1(1+\mu_1)}{\left(1-\mu_2\right)(1-\mu_1)^2}\right)L^2\eta^2\right]\\
\le & \left(\frac{1-\mu_1}{1-\mu_1^t}\right)\frac{\sigma^2}{b}+\frac{\mu_1(1+\mu_1)L^2\eta^2}{\left(1-\mu_2\right)(1-\mu_1)(1-\mu_1^t)}.
\end{align*}
Then by Lemma~\ref{lem:mut}, it follows that: \begin{align*}
\frac{1}{T}\sum_{t=1}^T\expect\left[\normF{E_t}^2\right] \le & \left(\frac{1}{T}\sum_{t=1}^T\frac{1}{1-\mu_1^t}\right)
\left(\frac{\sigma^2 (1-\mu_1)}{b}+\frac{\mu_1(1+\mu_1)L^2\eta^2}{(1-\mu_1)(1-\mu_2)}\right)\\
\le & \left(1+\frac{\mu_1}{(1-\mu_1)T}-\frac{1}{T\ln\mu_1}\ln\left(\frac{1-\mu_1^T}{1-\mu_1}\right)\right)
\left(\frac{\sigma^2 (1-\mu_1)}{b}+\frac{\mu_1(1+\mu_1)L^2\eta^2}{(1-\mu_1)(1-\mu_2)}\right).
\end{align*}
By Cauchy-Schwarz inequality and Jensen's inequality, \begin{align}\label{eq:sumEt}
\frac{1}{T}\sum_{t=1}^T\expect\left[\norm{E_t}_*\right]\le &\sqrt{\frac{r}{T}\sum_{t=1}^T\expect\left[\normF{E_t}^2\right]}\nonumber\\
\le & \sqrt{1+\frac{\mu_1}{(1-\mu_1)T}-\frac{1}{T\ln\mu_1}\ln\left(\frac{1-\mu_1^T}{1-\mu_1}\right)} 
\left(\frac{\sigma\sqrt{r(1-\mu_1)}}{\sqrt{b}}+L\eta 
\sqrt{\frac{r\mu_1(1+\mu_1)}{\left(1-\mu_2\right)(1-\mu_1)}}\right).
\end{align}

\par\noindent\textbf{Step 4: Bounding $\expect[\sqrt{\hat v_t}]$.}
For $\hat v_t,$ by Minkowski inequality and Jensen's inequality, \begin{align*}
\expect\left[\sqrt{\hat v_t}\right] = & \expect\left[\sqrt{\sum_{\tau=1}^tw_{2,t,\tau}\normF{G_\tau}^2}\right]\\
\le & \expect\left[\sqrt{\sum_{\tau=1}^tw_{2,t,\tau}\normF{\Grad\L(\Theta_{\tau-1})-G_\tau}^2}\right]
+\expect\left[\sqrt{\sum_{\tau=1}^tw_{2,t,\tau}\normF{\Grad\L(\Theta_{\tau-1})}^2}\right]\\
\le &  \frac{\sigma}{\sqrt{b}}+\expect\left[\normF{\Grad\L(\Theta_{t-1})}\right]
+\expect\left[\sqrt{\sum_{\tau=1}^tw_{2,t,\tau}\normF{\Grad\L(\Theta_{\tau-1})-\Grad\L(\Theta_{t-1})}^2}\right]\\
\le & \frac{\sigma}{\sqrt{b}}
+\expect\left[\normF{\Grad\L(\Theta_{t-1})}\right]+\eta L\expect\left[\sqrt{\sum_{\tau=1}^tw_{2,t,\tau}\left(\sum_{\tau=1}^t\alpha_s\right)^2}\right]\\
\stackrel{\eqref{eq:alphats}}{\le} & \frac{\sigma}{\sqrt{b}}+\expect\left[\normF{\Grad\L(\Theta_{t-1})}\right]+\eta L\sqrt{\frac{1-\mu_1}{1-\mu_2^t}}\sqrt{\sum_{\tau=1}^t\mu_2^{t-\tau}\left(t-\tau\right)^2}\\
\le & \frac{\sigma}{\sqrt{b}}+\expect\left[\normF{\Grad\L(\Theta_{t-1})}\right]
+\eta L\sqrt{\frac{1-\mu_1}{1-\mu_2^t}}\sqrt{\sum_{\tau=1}^\infty\mu_2^{t-\tau}\left(t-\tau\right)^2}\\
\le & \frac{\sigma}{\sqrt{b}}+\expect\left[\normF{\Grad\L(\Theta_{t-1})}\right]
+\eta L\sqrt{\frac{1-\mu_1}{1-\mu_2^t}}\sqrt{\sum_{\tau=1}^\infty\mu_2^{\tau}\tau^2}\\
\le & \frac{\sigma}{\sqrt{b}}+\expect\left[\normF{\Grad\L(\Theta_{t-1})}\right]+ \eta L 
\sqrt{\frac{(1-\mu_1)\mu_2(1+\mu_2)}{(1-\mu_2^t)(1-\mu_2)^3}}\\
\le & \frac{\sigma}{\sqrt{b}}+\expect\left[\normF{\Grad\L(\Theta_{t-1})}\right]+ \sqrt{2}\eta L 
\sqrt{\frac{(1-\mu_1)}{(1-\mu_2^t)(1-\mu_2)^3}}.
\end{align*}
Hence, \begin{equation}\label{eq:svt}
\expect\left[\sqrt{\hat v_t}+\frac{\epsilon}{\sqrt{1-\mu_2^t}}\right]\le \expect\left[\norm{\Grad\L(\Theta_{t-1})}\right]+a_t,
\end{equation} where \begin{equation*}
a_t: = \frac{\sigma}{\sqrt{b}}
+ \sqrt{2}\eta L \sqrt{\frac{(1-\mu_1)}{(1-\mu_2^t)(1-\mu_2)^3}}+\frac{\epsilon}{\sqrt{1-\mu_2^t}}.
\end{equation*} Then by Lemma~\ref{lem:mutsqrt},\begin{equation}\label{eq:sctsum}
\frac{1}{T}\sum_{t=1}^T a_t\le \frac{\sigma}{T}\sum_{t=1}^T \frac{1}{\sqrt{b}}+\left(\sqrt{2}\eta L \sqrt{\frac{(1-\mu_1)}{(1-\mu_2)^3}}+\epsilon\right)
\left(1-\frac{2\ln(1+\sqrt{1-\mu_2^T})}{T\ln\mu_2}\right).
\end{equation}

\par\noindent\textbf{Step 5: Relating $\expect\left[\alpha_t\normF{\hat M_t}\right]$ to $\expect\left[\normF{\Grad\L(\Theta_{t-1})}\right]$.}
By Cauchy-Schwarz inequality, \begin{align*}
\left(\expect\left[\normF{\hat M_t}\right]\right)^2= & \left(\expect\left[\frac{\normF{\hat M_t}}
{\left(\sqrt{\hat v_t}+\epsilon_t\right)^{\frac{1}{2}}}
\cdot \left(\sqrt{\hat v_t}+\epsilon_t\right)^{\frac{1}{2}}\right]\right)^2 \\
\le & \expect\left[\alpha_t\normF{\hat M_t}\right]\expect\left[\sqrt{\hat v_t}+\epsilon_t\right],
\end{align*}
where $\epsilon_t:=\epsilon/\sqrt{1-\mu_2^t}.$ Combining the above with \eqref{eq:svt} gives: \[
\expect\left[\alpha_t\normF{\hat M_t}\right]\ge \frac{\left(\expect\left[\normF{\hat M_t}\right]\right)^2}
{\expect\left[\normF{\Grad\L(\Theta_{t-1})}\right]+a_t}
\ge \frac{\left(\expect\left[\normF{\Grad\L(\Theta_{t-1})}\right]-
\expect\left[\normF{E_t}\right]\right)^2}
{\expect\left[\normF{\Grad\L(\Theta_{t-1})}\right]+a_t}.
\]
Rearranging the terms gives: \[
\left(\expect\left[\normF{\Grad\L(\Theta_{t-1})}\right]\right)^2-
\left(2\expect\left[\normF{E_t}\right]+\expect\left[\alpha_t\normF{\hat M_t}\right]\right)
\expect\left[\normF{\Grad\L(\Theta_{t-1})}\right]
-a_t\expect\left[\alpha_t\normF{\hat M_t}\right] +\expect\left[\normF{E_t}\right]^2\le 0.
\]
Then solving for $\expect\left[\normF{\Grad\L(\Theta_{t-1})}\right]$ gives: \begin{align*}
&\expect\left[\normF{\Grad\L(\Theta_{t-1})}\right]\\
\le & \frac{2\expect\left[\normF{E_t}\right]+\expect\left[\alpha_t\normF{\hat M_t}\right]
+\sqrt{\left(2\expect\left[\normF{E_t}\right]+\expect\left[\alpha_t\normF{\hat M_t}\right]\right)^2
+4a_t\expect\left[\alpha_t\normF{\hat M_t}\right]-4\expect\left[\normF{E_t}\right]^2}}{2}\\
\le & \expect\left[\normF{E_t}\right]+\frac{1}{2}\expect\left[\alpha_t\normF{\hat M_t}\right]
+\sqrt{\frac{1}{4}\expect\left[\alpha_t\normF{\hat M_t}\right]^2
+\left(\expect\left[\normF{E_t}\right]+a_t\right)\expect\left[\alpha_t\normF{\hat M_t}\right]}\\
\le & \expect\left[\normF{E_t}\right]+\expect\left[\alpha_t\normF{\hat M_t}\right]
+\sqrt{\left(\expect\left[\normF{E_t}\right]+a_t\right)\expect\left[\alpha_t\normF{\hat M_t}\right]}
\end{align*} 

\par\noindent\textbf{Step 6: Deriving the convergence rate.}
By Cauchy-Schwarz inequality, \begin{align*}
&\frac{1}{T}\sum_{t=1}^T\expect\left[\normF{\Grad\L(\Theta_{t-1})}\right]\\
\le  &
\frac{1}{T}\sum_{t=1}^T\expect\left[\normF{E_t}\right]+
\frac{1}{T}\sum_{t=1}^T\expect\left[\alpha_t\normF{\hat M_t}\right] +
\sqrt{\frac{1}{T}\sum_{t=1}^T\left(\expect\left[\normF{E_t}\right]+a_t\right)}
\sqrt{\frac{1}{T}\sum_{t=1}^T\expect\left[\alpha_t\normF{\hat M_t}\right]}.
\end{align*}
Combining the above with \eqref{eq:totalM} and \eqref{eq:sctsum} gives: 
\begin{align*}
&\frac{1}{T}\sum_{t=1}^T\expect\left[\normF{\Grad\L(\Theta_{t-1})}\right]\\
\le &\frac{1}{T}\sum_{t=1}^T\expect\left[\norm{E_t}_*\right]+  \frac{\Delta}{\eta T}+\frac{\eta L (1-\mu_1)}{2 (1-\mu_2)}
+ \frac{1}{T}\sqrt{\frac{1-\mu_1}{1-\mu_2}}\sum_{t=1}^T\expect\left[\norm{E_t}_*\right] \\
& +\sqrt{\frac{1}{T}\sum_{t=1}^T\left(\expect\left[\normF{E_t}\right]+a_t\right)} \cdot \sqrt{\frac{\Delta}{\eta T}
+ \frac{1}{T}\sqrt{\frac{1-\mu_1}{1-\mu_2}}\sum_{t=1}^T\expect\left[\norm{E_t}_*\right] 
+\frac{\eta L (1-\mu_1)}{2 (1-\mu_2)}}\\
\le & \frac{1}{T}\sum_{t=1}^T\expect\left[\norm{E_t}_*\right]+  \frac{\Delta}{\eta T}+\frac{\eta L (1-\mu_1)}{2 (1-\mu_2)}
+ \frac{1}{T}\sqrt{\frac{1-\mu_1}{1-\mu_2}}\sum_{t=1}^T\expect\left[\norm{E_t}_*\right] \\
& +\sqrt{\frac{1}{T}\sum_{t=1}^T \expect\left[\normF{E_t}\right]+
\frac{\sigma}{\sqrt{b}}+\left(\sqrt{2}\eta L \sqrt{\frac{(1-\mu_1)}{(1-\mu_2)^3}}+\epsilon\right)
\left(1-\frac{2\ln(1+\sqrt{1-\mu_2^T})}{T\ln\mu_2}\right)}\\
& \cdot \sqrt{\frac{\Delta}{\eta T}
+ \frac{1}{T}\sqrt{\frac{1-\mu_1}{1-\mu_2}}\sum_{t=1}^T\expect\left[\norm{E_t}_*\right] 
+\frac{\eta L (1-\mu_1)}{2 (1-\mu_2)}}.
\end{align*}
In particular, for large $T>0,$ if choosing
$\eta=\bigO(T^{-\frac{3}{4}})$, $1-\mu_1=\Theta(T^{-\frac{1}{2}})$, 
$1-\mu_2=\Theta(T^{-\frac{1}{2}}),$ $0\le \mu_1\le \mu_2<1$, and $\epsilon =\bigO(T^{-\frac{1}{2}})$,
then, by \eqref{eq:sumEt},\[
\frac{1}{T}\sum_{t=1}^T\expect\left[\norm{E_t}_*\right]
\le\bigO\left(T^{-\frac{1}{4}}\right),
\]
and it follows that:
\[
\frac{1}{T}\sum_{t=1}^T\expect\left[\normF{\Grad\L(\Theta_{t-1})}\right]\le \bigO\left(T^{-\frac{1}{4}}+\sqrt{\sigma}b^{-\frac{1}{4}}T^{-\frac{1}{8}}\right),
\] where $b$ is the batch size. 
The proof is thus completed.
\end{proof}

\section{Proof of Theorem~\ref{thm:ddm}}\label{pf:ddm}
This section contains the detailed proof of Theorem~\ref{thm:ddm} for the convergence of \NAMOD in the deterministic setting.
\begin{proof}
\par\noindent\textbf{Step 1: A uniform upper bound on stepsize.}
For each $t\ge 0$ and $\tau\le t,$  define: \[
w_{1,t,\tau}:=\frac{1-\mu_1}{1-\mu_1^t}\mu_1^{t-\tau} ~\textrm{ and }~ w_{2,t,\tau}:=\frac{1-\mu_2}{1-\mu_2^t}\mu_2^{t-\tau}.
\]
It satisfies that $\sum_{\tau=1}^tw_{1,t,\tau}=\sum_{\tau=1}^tw_{2,t,\tau}=1.$
For each $j=1,\cdots, n$ and $t>0,$ let $M_t^j$ denote the $j$-th column of $M_t$, and $\bv_t^j$ denote the $j$-th element of $\bv_t.$
Write $\hat M_t:= \frac{1}{1-\mu_1^t} M_t,$ $\hat \bv_t:=\frac{1}{1-\mu_2^t}\bv_t,$ and \[
D_t:= \diag\left(\min\left\{\max\left\{\bd_t,c \bar d_t \mathbf{1}\right\},\frac{1}{c}\bar d_t \mathbf{1}\right\}\right),
\] where $\bd_t$ and $\bar d_t$ are given in Algorithm~\ref{alg:namod}. 
Then: \[
\hat M_t=\frac{1-\mu_1}{1-\mu_1^t}\sum_{\tau=1}^t\mu_1^{t-\tau}\Grad\L(\Theta_{\tau-1})=\sum_{\tau=1}^t w_{1,t,\tau}\Grad\L(\Theta_{\tau-1}),
\] and \[
\hat \bv_t=\frac{1-\mu_2}{1-\mu_2^t}\sum_{\tau=1}^t\mu_2^{t-\tau}\N_c(\Grad\L(\Theta_{\tau-1}))\odot 
\N_c(\Grad\L(\Theta_{\tau-1}))=\sum_{\tau=1}^tw_{2,t,\tau}\N_c(\Grad\L(\Theta_{\tau-1}))\odot \N_c(\Grad\L(\Theta_{\tau-1})).
\]
By Lemma~\ref{lem:snr}, \begin{equation}\label{eq:ddtj}
\comp{D_t}_{jj}\le \frac{\norm{\hat M^j_t}}{\sqrt{\hat \bv^j_t}}\le \sqrt{\frac{1-\mu_1}{1-\mu_2}}, \quad\forall j.
\end{equation}
Write: \begin{equation}\label{eq:ddtm}
d_{t, \max} := \max_j \comp{D_t}_{jj}, 
~\textrm{ and } d_{t, \min} := \min_j \comp{D_t}_{jj}.
\end{equation} Then the condition number of $D_t$ is given by: \begin{equation}\label{eq:ddcond}
\kappa\left(D_t\right) =\frac{d_{t, \max}}{d_{t, \min}}\le \min\left\{\kappa_t, \frac{1}{c^2}\right\},
\end{equation}
where $c\in (0,1]$ is the fixed clamping hyperparameter, $\kappa_t:=\kappa\left(\diag\left(\bd_t\right)\right)$ denotes the condition number of $\diag\left(\bd_t\right).$

\par\noindent\textbf{Step 2: Descent inequality and averaging.}
By \citep[Lemma B.1]{zhang2025adagrad},
\begin{align*}
 & \L(\Theta_{t})-\L(\Theta_{t-1})\\
\le & -\dotP{\Grad\L(\Theta_{t-1})}{\eta O_t D_t} +\frac{ L}{2}\eta^2\norm{D_t}_2^2\\
\le &-\eta d_{t,\min}\norm{\Grad\L(\Theta_{t-1})}_* +2\eta d_{t,\max} \norm{\Grad\L(\Theta_{t-1})-\hat M_t}_*
+\frac{ L}{2}\eta^2\norm{D_t}_2^2\\
= & -\eta\max\{\kappa_t^{-1}, c^2\}d_{t,\max}\norm{\Grad\L(\Theta_{t-1})}_*
+2\eta d_{t,\max} \norm{\Grad\L(\Theta_{t-1})-\hat M_t}_*
+\frac{ L}{2}\eta^2\norm{D_t}_2^2.
\end{align*}
Rearranging the terms gives: \begin{align}\label{eq:dphiTm}
\frac{1}{T}\sum_{t=1}^T d_{t,\max}\norm{\Grad\L(\Theta_{t-1})}_*
\le & \frac{\Delta}{\eta T c^2}+\frac{2}{Tc^2}\sum_{t=1}^T d_{t,\max}\norm{\Grad\L(\Theta_{t-1})-\hat M_t}_*+\frac{\eta L}{2T c^2}\sum_{t=1}^Td_{t,\max}^2\nonumber\\
\le &  \frac{\Delta}{\eta Tc^2}+\frac{2}{Tc^2}\sqrt{\frac{1-\mu_1}{1-\mu_2}}\sum_{t=1}^T\norm{\Grad\L(\Theta_{t-1})-\hat M_t}_*
+\frac{\eta L}{2c^2}\left(\frac{1-\mu_1}{1-\mu_2}\right).
\end{align}

\par\noindent\textbf{Step 3: Bounding $\|\hat M_t-\nabla\L(\Theta_{t-1})\|_*$.}
The difference between the scaled momentum $\hat M_t$ and the gradient $\Grad\L(\Theta_{t-1})$ can be bounded by: \begin{align*}
\norm{\hat M_t-\Grad\L(\Theta_{t-1})}_*\le & \sum_{\tau=1}^tw_{1,t,\tau}\norm{\Grad\L(\Theta_{\tau-1})-\Grad\L(\Theta_{t-1})}_*\\
\le &  \sum_{\tau=1}^tw_{1,t,\tau}L\norm{\Theta_{\tau-1}-\Theta_{t-1}}_2\\
\le & \sum_{\tau=1}^tw_{1,t,\tau}L\norm{\sum_{s=\tau}^{t-1}\eta O_s D_s}_2\\
\le &\sum_{\tau=1}^tw_{1,t,\tau}\eta L\left(\sum_{s=\tau}^{t-1}d_{s,\max}\right).
\end{align*} Then by \eqref{eq:alphat} and the definition of $w_{1,t,\tau}$, \begin{align}\label{eq:dMtilde}
\norm{\hat M_t-\Grad\L(\Theta_{t-1})}_*\le & \sum_{\tau=1}^tw_{1,t,\tau}\eta L\left(t-\tau\right)\sqrt{\frac{1-\mu_1}{1-\mu_2}}\nonumber\\
= &  \sum_{\tau=1}^t \eta L\left(t-\tau\right)\mu_1^{t-\tau}\frac{1-\mu_1}{1-\mu_1^t}\sqrt{\frac{1-\mu_1}{1-\mu_2}}\nonumber\\
\le & \eta L \frac{1-\mu_1}{1-\mu_1^{t}}\sqrt{\frac{1-\mu_1}{1-\mu_2}}\sum_{j=1}^\infty j\mu_1^j\nonumber\\
\le & \eta L\sqrt{\frac{1-\mu_1}{1-\mu_2}}\frac{\mu_1}{(1-\mu_1)^2}.
\end{align}

\par\noindent\textbf{Step 4: Bounding $\sqrt{\sum_{j=1}^n \hat \bv_t^j}$.}
For $\hat \bv_t$, it satisfies that: \begin{align*}
\sum_{j=1}^n\hat \bv_t^j = & \sum_{\tau=1}^t w_{2,t,\tau}\normF{\Grad\L(\Theta_{\tau-1})}^2\\
\le & \sum_{\tau=1}^t w_{2,t,\tau}\left(\normF{\Grad\L(\Theta_{t-1}}+L\norm{\Theta_{\tau-1}-\Theta_{t-1}}_2\right)^2\\
\le & \sum_{\tau=1}^t w_{2,t,\tau}\left(\normF{\Grad\L(\Theta_{t-1}}+\eta L\left(\sum_{s=\tau}^{t-1} d_{s,\max}\right)\right)^2\\
\end{align*}
Then by \eqref{eq:alphat} and the definition of $w_{2,t,\tau}$, \begin{align*}
\sum_{j=1}^n\hat \bv_t^j \le &\sum_{\tau=1}^t w_{2,t,\tau}\left(\normF{\Grad\L(\Theta_{t-1}}+\eta L\left(t-\tau\right)\sqrt{\frac{1-\mu_1}{1-\mu_2}}\right)^2\\
= &\normF{\Grad\L(\Theta_{t-1})}^2+2\eta L\normF{\Grad\L(\Theta_{t-1})}\sqrt{\frac{1-\mu_1}{1-\mu_2}}\sum_{\tau=1}^t w_{2,t,\tau}(t-\tau)\\
& + \eta^2 L^2\left(\frac{1-\mu_1}{1-\mu_2}\right)\sum_{\tau=1}^t w_{2,t,\tau}(t-\tau)^2.
\end{align*}
By Cauchy-Schwarz inequality, it follows that: \begin{align*}
\sum_{j=1}^n\hat \bv_t^j  \le & \normF{\Grad\L(\Theta_{t-1})}^2+2\eta L\normF{\Grad\L(\Theta_{t-1})}\sqrt{\frac{1-\mu_1}{1-\mu_2}}\sqrt{\sum_{\tau=1}^t w_{2,t,\tau}(t-\tau)^2}\\
& + \eta^2 L^2\left(\frac{1-\mu_1}{1-\mu_2}\right)\left(\sum_{\tau=1}^t w_{2,t,\tau}(t-\tau)^2\right)\\
\le & \left(\normF{\Grad\L(\Theta_{t-1})}+\eta L\sqrt{\frac{1-\mu_1}{1-\mu_2}}\sqrt{\sum_{\tau=1}^t w_{2,t,\tau}(t-\tau)^2}\right)^2\\
=& \left(\normF{\Grad\L(\Theta_{t-1})}+a_t\right)^2,
\end{align*} where \[
a_t:= \eta L\sqrt{\frac{1-\mu_1}{1-\mu_2}}\sqrt{\sum_{\tau=1}^t w_{2,t,\tau}(t-\tau)^2}.
\] Since: \begin{align}\label{eq:dcsquared}
a_t^2 \le & \eta^2 L^2\left(\frac{1-\mu_1}{1-\mu_2}\right)
\left(\frac{1-\mu_2}{1-\mu_2^t}\right)\left(\sum_{\tau=1}^t \mu_2^{t-\tau}(t-\tau)^2\right)\nonumber\\
\le & \eta^2 L^2\left(\frac{1-\mu_1}{1-\mu_2}\right)\left(\sum_{\tau=1}^\infty \mu_2^{\tau}\tau^2\right)\nonumber\\
\le & \eta^2 L^2\left(\frac{1-\mu_1}{1-\mu_2}\right)\left(\frac{\mu_2(1+\mu_2)}{(1-\mu_2)^3}\right):=a^2.
\end{align} Hence, \begin{equation}\label{eq:dvtilde}
\sqrt{\sum_{j=1}^n\hat \bv_t^j}\le \normF{\Grad\L(\Theta_{t-1})}+a_t
\le \normF{\Grad\L(\Theta_{t-1})}+a.
\end{equation}

\par\noindent\textbf{Step 5: Lower-bounding $d_{t,\max}\|\nabla\L(\Theta_{t-1})\|_*$.}
Combining \eqref{eq:dMtilde} and \eqref{eq:dvtilde} gives: \begin{align*}
 & d_{t,\max}\norm{\Grad\L(\Theta_{t-1})}_*\\
 \ge & \frac{\norm{\hat M_t}\norm{\Grad\L(\Theta_{t-1}}_*}
{\sqrt{\sum_{j=1}^n \hat \bv_t^j}+n\epsilon/\sqrt{1-\mu_2^t}}\\
\ge & \frac{\left(\normF{\Grad\L(\Theta_{t-1}}-\norm{\hat M_t-\Grad\L(\Theta_{t-1})}\right)\norm{\Grad\L(\Theta_{t-1}}_*}
{\normF{\Grad\L(\Theta_{t-1})}+n\epsilon/\sqrt{1-\mu_2^t}+a}\\
\ge & \frac{\normF{\Grad\L(\Theta_{t-1}}^2}{\normF{\Grad\L(\Theta_{t-1})}+n\epsilon/\sqrt{1-\mu_2^t}+a} 
-\norm{\hat M_t-\Grad\L(\Theta_{t-1})}\\
\ge & \frac{\normF{\Grad\L(\Theta_{t-1}}^2}{\normF{\Grad\L(\Theta_{t-1})}+n\epsilon/\sqrt{1-\mu_2^t}}\left(1-\frac{a}
{\normF{\Grad\L(\Theta_{t-1})}+n\epsilon/\sqrt{1-\mu_2^t}+a}\right)-\norm{\hat M_t-\Grad\L(\Theta_{t-1})}\\
\ge & \frac{\normF{\Grad\L(\Theta_{t-1}}^2}{\normF{\Grad\L(\Theta_{t-1})}+n\epsilon/\sqrt{1-\mu_2^t}}-a-\norm{\hat M_t-\Grad\L(\Theta_{t-1})}\\
\ge & \frac{\normF{\Grad\L(\Theta_{t-1}}^2}{\normF{\Grad\L(\Theta_{t-1})}+n\epsilon/\sqrt{1-\mu_2^t}}-a-\eta L\sqrt{\frac{1-\mu_1}{1-\mu_2}}\frac{\mu_1}{(1-\mu_1)^2}\\
\ge & \frac{\normF{\Grad\L(\Theta_{t-1}}^2}{\normF{\Grad\L(\Theta_{t-1})}+n\tilde\epsilon}-a-\eta L\sqrt{\frac{1-\mu_1}{1-\mu_2}}\frac{\mu_1}{(1-\mu_1)^2}
\end{align*}
where $\tilde \epsilon := \epsilon/\sqrt{1-\mu_2}$ and $a$ is given in \eqref{eq:dcsquared}.

\par\noindent\textbf{Step 6: Upper-bounding the gradient norm and deriving rate.}
Now define $\phi_{n\tilde \epsilon}(x):=\frac{x^2}{x+n\tilde \epsilon},$ it then follows from \eqref{eq:dphiTm} and \eqref{eq:dMtilde} that: \begin{align*}
\frac{1}{T}\sum_{t=1}^T \phi_{n\tilde \epsilon}(\normF{\Grad\L(\Theta_{t-1}})\le & \frac{1}{T}\sum_{t=1}^T 
d_{t,\max}\norm{\Grad\L(\Theta_{t-1}}_*+a+\eta L\sqrt{\frac{1-\mu_1}{1-\mu_2}}\frac{\mu_1}{(1-\mu_1)^2}\\
\le & \frac{\Delta}{\eta Tc^2}+\frac{2}{Tc^2}\sqrt{\frac{1-\mu_1}{1-\mu_2}}\sum_{t=1}^T\norm{\Grad\L(\Theta_{t-1})-\hat M_t}_*
+\frac{\eta L}{2c^2}\left(\frac{1-\mu_1}{1-\mu_2}\right)\\
& +a+\eta L\sqrt{\frac{1-\mu_1}{1-\mu_2}}\frac{\mu_1}{(1-\mu_1)^2}\\
\le &  \frac{\Delta}{\eta Tc^2}+\frac{2}{Tc^2}\sqrt{\frac{1-\mu_1}{1-\mu_2}}\sum_{t=1}^T\norm{\Grad\L(\Theta_{t-1})-\hat M_t}_*
+\frac{\eta L}{2c^2}\left(\frac{1-\mu_1}{1-\mu_2}\right)\\
& +a+\eta L\sqrt{\frac{1-\mu_1}{1-\mu_2}}\frac{\mu_1}{(1-\mu_1)^2}\\
\le & \frac{\Delta}{\eta Tc^2}+\left(\frac{1-\mu_1}{1-\mu_2}\right)\frac{2\eta L\mu_1}{(1-\mu_1)^2c^2}
+\frac{\eta L}{2c^2}\left(\frac{1-\mu_1}{1-\mu_2}\right)\\
& +\eta L\frac{\sqrt{\mu_2(1-\mu_1)(1+\mu_2)}}{(1-\mu_2)^2}+\sqrt{\frac{1-\mu_1}{1-\mu_2}}\frac{\eta L\mu_1}{(1-\mu_1)^2}\\
=& \frac{\Delta}{\eta Tc^2}+\frac{\eta L C_{\mu}}{c^2},
\end{align*} where \[
C_\mu:= \left(\frac{1-\mu_1}{1-\mu_2}\right)\frac{2\mu_1}{(1-\mu_1)^2}
+\frac{1}{2}\left(\frac{1-\mu_1}{1-\mu_2}\right)
+\frac{c^2\sqrt{\mu_2(1-\mu_1)(1+\mu_2)}}{(1-\mu_2)^2}+\sqrt{\frac{1-\mu_1}{1-\mu_2}}\frac{\mu_1c^2}{(1-\mu_1)^2}
\] is a constant that depends on $\mu_1$ and $\mu_2$. Then by Lemma~\ref{lem:phi_eps}, \begin{align*}
\frac{1}{T}\sum_{t=1}^T \normF{\Grad\L(\Theta_{t-1}}\le & \frac{1}{T}\sum_{t=1}^T \phi_{n\tilde\epsilon}\left(\normF{\Grad\L(\Theta_{t-1}}\right)
+(n \tilde \epsilon)^{\frac{1}{2}}\frac{1}{T}\sum_{t=1}^T \sqrt{\phi_{n\tilde\epsilon}\left(\normF{\Grad\L(\Theta_{t-1}}\right)}\\
\le & \frac{1}{T}\sum_{t=1}^T \phi_{n\tilde\epsilon}\left(\normF{\Grad\L(\Theta_{t-1}}\right)
+\frac{(n\tilde\epsilon)^{\frac{1}{2}}}{\sqrt{T}}\sqrt{\sum_{t=1}^T \phi_{n\tilde\epsilon}\left(\normF{\Grad\L(\Theta_{t-1}}\right)}\\
\le & \frac{\Delta}{\eta Tc^2}+\frac{\eta L C_{\mu}}{c^2}  + (n\tilde \epsilon)^{\frac{1}{2}}\sqrt{\frac{\Delta}{\eta Tc^2}
+\frac{\eta L C_{\mu}}{c^2}}.
\end{align*}

In particular, if choosing 
$\eta=\bigO(T^{-\frac{1}{2}})$, $1-\mu_1=\Theta(1)$, 
$1-\mu_2=\Theta(1),$ $0\le \mu_1\le \mu_2<1$, $\epsilon =\bigO(T^{-\frac{1}{2}}n^{-1})$,
and $c=\Theta(1),$ then: \[
\frac{1}{T}\sum_{t=1}^T\norm{\Grad\L(\Theta_{t-1})}\le\bigO\left(T^{-\frac{1}{2}}\right)
\]
for large $T>0$. The proof is thus completed.
\end{proof}

\section{Proof of Theorem~\ref{thm:dsm}}\label{pf:dsm}
This section contains the detailed proof of Theorem~\ref{thm:dsm} for the convergence of \NAMOD in the stochastic setting.
\begin{proof}
\par\noindent\textbf{Step 1: A uniform upper bound on stepsize.}
For each $t\ge 0$ and $\tau\le t,$  define: \[
w_{1,t,\tau}:=\frac{1-\mu_1}{1-\mu_1^t}\mu_1^{t-\tau} ~\textrm{ and }~ w_{2,t,\tau}:=\frac{1-\mu_2}{1-\mu_2^t}\mu_2^{t-\tau}.
\]
It satisfies that $\sum_{\tau=1}^tw_{1,t,\tau}=\sum_{\tau=1}^tw_{2,t,\tau}=1.$
For each $j=1,\cdots, n$ and $t>0,$ let $M_t^j$ denote the $j$-th column of $M_t$, and $\bv_t^j$ denote the $j$-th element of $\bv_t.$
Write $\hat M_t:= \frac{1}{1-\mu_1^t} M_t,$ $\hat \bv_t:=\frac{1}{1-\mu_2^t}\bv_t,$ and \[
D_t:= \diag\left(\min\left\{\max\left\{\bd_t,c \bar d_t \mathbf{1}\right\},\frac{1}{c}\bar d_t \mathbf{1}\right\}\right),
\] where  $c\in (0,1]$ is the fixed clamping hyperparameter, $\bd_t$ and $\bar d_t$ are given in Algorithm~\ref{alg:namod}. 
Then: \[
\hat M_t=\frac{1-\mu_1}{1-\mu_1^t}\sum_{\tau=1}^t\mu_1^{t-\tau}G_\tau=\sum_{\tau=1}^t w_{1,t,\tau}G_\tau,
\] and \[
\hat \bv_t=\frac{1-\mu_2}{1-\mu_2^t}\sum_{\tau=1}^t\mu_2^{t-\tau}\N_c(G_\tau)\odot \N_c(G_\tau)=\sum_{\tau=1}^tw_{2,t,\tau}\N_c(G_\tau)\odot \N_c(G_\tau).
\]
By Lemma~\ref{lem:snr}, \begin{equation}\label{eq:dtj}
\comp{D_t}_{jj}\le \frac{\norm{\hat M^j_t}}{\sqrt{\hat \bv^j_t}}\le \sqrt{\frac{1-\mu_1}{1-\mu_2}}, \quad\forall j.
\end{equation}
Write: \begin{equation}\label{eq:dtm}
d_{t, \max} := \max_j \comp{D_t}_{jj}, 
~\textrm{ and } d_{t, \min} := \min_j \comp{D_t}_{jj}.
\end{equation} Then the condition number of $D_t$ is given by: \begin{equation}\label{eq:dcond}
\kappa\left(D_t\right) =\frac{d_{t, \max}}{d_{t, \min}}\le \min\left\{\kappa_t, \frac{1}{c^2}\right\},
\end{equation}
where $\kappa_t:=\kappa\left(\diag(\bd_t)\right)$ denotes the condition number of $\diag(\bd_t)$.

\par\noindent\textbf{Step 2: Expected descent inequality and averaging.}
Let $\expect_t[\cdot]:=\expect[\cdot|\Theta_{t-1}]$ denote the conditional expectation given the previous iterates
$\Theta_0,\cdots,\Theta_{t-1},$ and write $E_t:=\hat M_t-\Grad\L(\Theta_{t-1})$. Then, by Lemma~\ref{lem:od},
\begin{align*}
&\expect_t\left[\L(\Theta_{t})-\L(\Theta_{t-1})\right]\\
\le &\expect_t\left[-\dotP{\Grad\L(\Theta_{t-1})}
{ \eta O_tD_t}\right]+  \frac{\eta^2L}{2}\expect_t\left[ \norm{D_t}_2^2\right]\\
= &\expect_t\left[-\dotP{\Grad\L(\Theta_{t-1})-\hat M_t}{ \eta O_tD_t}\right] 
-\expect_t\left[\eta\dotP{\hat M_t}{O_tD_t}\right]+
\frac{\eta^2 L}{2}\expect_t\left[\norm{D_t}_2^2\right]\\
\le & \Biggl(\expect_t\left[ \eta d_{t,\max}\norm{\Grad\L(\Theta_{t-1})-\hat M_t}_*\right] 
-\expect_t\left[\eta d_{t,\min}\norm{\hat M_t}_*\right] \Biggr)
+\frac{\eta^2 L}{2}\expect_t\left[d_{t,\max}^2\right]\\
\le &  -\expect_t\left[\eta d_{t,\min}\norm{\hat M_t}_*\right] 
+\expect_t\left[ \eta d_{t,\max}\norm{E_t}_*\right] 
+\frac{\eta^2 L}{2}\expect_t\left[ d_{t,\max}^2\right],
\end{align*}
where $d_{t,\max}$ and $d_{t,\min}$ are defined in \eqref{eq:dtm}.
Rearranging the terms gives: \[
\expect_t\left[d_{t,\min}\norm{\hat M_t}_*\right]\le \expect_t\left[\L(\Theta_{t-1})-\L(\Theta_{t})\right]+
\expect_t\left[d_{t,\max}\norm{E_t}_*\right] 
+\frac{\eta L}{2}\expect_t\left[d_{t,\max}^2\right].
\]
Then by the law of total expectation and \eqref{eq:dtj},
\begin{align}\label{eq:dtotalM}
\frac{1}{T}\sum_{t=1}^T\expect\left[d_{t,\min}\norm{\hat M_t}_*\right]\nonumber
\le & \frac{\Delta}{\eta T}+ \frac{1}{T}\sum_{t=1}^T\expect\left[d_{t,\max}\norm{E_t}_*\right] 
+\frac{\eta L}{2T}\sum_{t=1}^T \expect\left[d_{t,\max}^2\right]\\
\le & \frac{\Delta}{\eta T}+ \frac{1}{T}\sqrt{\frac{1-\mu_1}{1-\mu_2}}\sum_{t=1}^T\expect\left[\norm{E_t}_*\right] 
+\frac{\eta L (1-\mu_1)}{2 (1-\mu_2)}.
\end{align}

\par\noindent\textbf{Step 3: Bounding the distance between bias-corrected momentum and true gradient.}
For each $t$, it satisfies: \begin{align*}
E_t = & \hat M_t - \expect\left[\hat M_t\right] + \expect\left[\hat M_t\right] -\Grad\L(\Theta_{t-1})\\
= & \sum_{\tau=1}^t w_{1,t,\tau} \left(G_\tau-\Grad\L(\Theta_{\tau-1})\right)+\sum_{\tau=1}^t w_{1,t,\tau} 
\left(\Grad\L(\Theta_{\tau-1})-\Grad\L(\Theta_{t-1})\right)
\end{align*} Hence, by \eqref{eq:dtj}, \begin{align*}
\expect\left[\normF{E_t}^2\right] \le & \expect\left[\normF{\sum_{\tau=1}^t w_{1,t,\tau} \left(G_\tau-\Grad\L(\Theta_{\tau-1})\right)}^2\right] 
+\normF{\sum_{\tau=1}^t w_{1,t,\tau} \left(\Grad\L(\Theta_{\tau-1})-\Grad\L(\Theta_{t-1})\right)}^2\\
\le & \left(\frac{1-\mu_1}{1-\mu_1^t}\right)^2\left(\sum_{\tau=1}^t\mu_1^{2(t-\tau)}\right)\frac{\sigma^2}{b}
+\sum_{\tau=1}^t w_{1,t,\tau}\normF{\Grad\L(\Theta_{\tau-1})-\Grad\L(\Theta_{t-1})}^2\\
\le & \left(\frac{1-\mu_1}{1-\mu_1^t}\right)^2\left(\frac{1-\mu_1^{2t}}{1-\mu_1^2}\right)\frac{\sigma^2}{b}
+\left(\frac{1-\mu_1}{1-\mu_1^t}\right)\sum_{\tau=1}^t\mu_1^{t-\tau}L^2\norm{\Theta_{\tau-1}-\Theta_{t-1}}_2^2\\
\le & \left(\frac{1-\mu_1}{1-\mu_1^t}\right)\left(\frac{1+\mu_1^{t}}{1+\mu_1}\right)\frac{\sigma^2}{b}
+ \left(\frac{1-\mu_1}{1-\mu_1^t}\right)L^2\eta^2\sum_{\tau=1}^t\mu_1^{t-\tau}\left(\sum_{s=\tau}^{t-1}\norm{D_s}_2\right)^2\\
\le & \left(\frac{1-\mu_1}{1-\mu_1^t}\right)\left[\left(\frac{1+\mu_1^{t}}{1+\mu_1}\right)\frac{\sigma^2}{b}
+ \left(\frac{1-\mu_1}{1-\mu_2}\right)L^2\eta^2\sum_{\tau=1}^t\mu_1^{t-\tau}\left(t-\tau\right)^2\right]\\
\le & \left(\frac{1-\mu_1}{1-\mu_1^t}\right)\left[\left(\frac{1+\mu_1^{t}}{1+\mu_1}\right)\frac{\sigma^2}{b}
+ \left(\frac{1-\mu_1}{1-\mu_2}\right)L^2\eta^2\sum_{\tau=1}^\infty\mu_1^{\tau}\tau^2\right]\\
= & \left(\frac{1-\mu_1}{1-\mu_1^t}\right)\left[\left(\frac{1+\mu_1^{t}}{1+\mu_1}\right)\frac{\sigma^2}{b}
+ \left(\frac{\mu_1(1+\mu_1)}{\left(1-\mu_2\right)(1-\mu_1)^2}\right)L^2\eta^2\right]\\
\le & \left(\frac{1-\mu_1}{1-\mu_1^t}\right)\frac{\sigma^2}{b}+\frac{\mu_1(1+\mu_1)L^2\eta^2}{\left(1-\mu_2\right)(1-\mu_1)(1-\mu_1^t)}.
\end{align*}
Then by Lemma~\ref{lem:mut}, it follows that: \begin{align*}
\frac{1}{T}\sum_{t=1}^T\expect\left[\normF{E_t}^2\right] \le & \left(\frac{1}{T}\sum_{t=1}^T\frac{1}{1-\mu_1^t}\right)
\left(\frac{\sigma^2 (1-\mu_1)}{b}+\frac{\mu_1(1+\mu_1)L^2\eta^2}{(1-\mu_1)(1-\mu_2)}\right)\\
\le & \left(1+\frac{\mu_1}{(1-\mu_1)T}-\frac{1}{T\ln\mu_1}\ln\left(\frac{1-\mu_1^T}{1-\mu_1}\right)\right)
\left(\frac{\sigma^2 (1-\mu_1)}{b}+\frac{\mu_1(1+\mu_1)L^2\eta^2}{(1-\mu_1)(1-\mu_2)}\right).
\end{align*}
By Cauchy-Schwarz inequality and Jensen's inequality, \begin{align}\label{eq:dsumEt}
\frac{1}{T}\sum_{t=1}^T\expect\left[\norm{E_t}_*\right]\le &\sqrt{\frac{r}{T}\sum_{t=1}^T\expect\left[\normF{E_t}^2\right]}\nonumber\\
\le & \sqrt{1+\frac{\mu_1}{(1-\mu_1)T}-\frac{1}{T\ln\mu_1}\ln\left(\frac{1-\mu_1^T}{1-\mu_1}\right)} 
\left(\frac{\sigma\sqrt{r(1-\mu_1)}}{\sqrt{b}}+L\eta 
\sqrt{\frac{r\mu_1(1+\mu_1)}{\left(1-\mu_2\right)(1-\mu_1)}}\right).
\end{align}

\par\noindent\textbf{Step 4: Bounding $\expect\!\left[\sqrt{\sum_{j=1}^n\hat \bv_t^j}\right]$ .}
For each $j$, let $G^j_\tau$ denote the $j$-th column of $G_\tau.$  By Minkowski inequality and Jensen's inequality,\begin{align*}
\expect\left[\sqrt{\sum_{j=1}^n\hat \bv^j_t}\right] = & \expect\left[\sqrt{\sum_{j=1}^n\sum_{\tau=1}^tw_{2,t,\tau}\norm{G^j_\tau}^2}\right]=
\expect\left[\sqrt{\sum_{\tau=1}^tw_{2,t,\tau}\normF{G_\tau}^2}\right]\\
\le & \expect\left[\sqrt{\sum_{\tau=1}^tw_{2,t,\tau}\normF{\Grad\L(\Theta_{\tau-1})-G_\tau}^2}\right]
+\expect\left[\sqrt{\sum_{\tau=1}^tw_{2,t,\tau}\normF{\Grad\L(\Theta_{\tau-1})}^2}\right]\\
\le &  \frac{\sigma}{\sqrt{b}}+\expect\left[\normF{\Grad\L(\Theta_{t-1})}\right]
+\expect\left[\sqrt{\sum_{\tau=1}^tw_{2,t,\tau}\normF{\Grad\L(\Theta_{\tau-1})-\Grad\L(\Theta_{t-1})}^2}\right]\\
\le & \frac{\sigma}{\sqrt{b}}
+\expect\left[\normF{\Grad\L(\Theta_{t-1})}\right]+\eta L\expect\left[\sqrt{\sum_{\tau=1}^tw_{2,t,\tau}\left(\sum_{s=\tau}^{t-1}
\norm{D_s}_2\right)^2}\right]\\
\stackrel{\eqref{eq:dtj}}{\le} & \frac{\sigma}{\sqrt{b}}+\expect\left[\normF{\Grad\L(\Theta_{t-1})}\right]+\eta L\sqrt{\frac{1-\mu_1}{1-\mu_2^t}}\sqrt{\sum_{\tau=1}^t\mu_2^{t-\tau}\left(t-\tau\right)^2}\\
\le & \frac{\sigma}{\sqrt{b}}+\expect\left[\normF{\Grad\L(\Theta_{t-1})}\right]
+\eta L\sqrt{\frac{1-\mu_1}{1-\mu_2^t}}\sqrt{\sum_{\tau=1}^\infty\mu_2^{t-\tau}\left(t-\tau\right)^2}\\
\le & \frac{\sigma}{\sqrt{b}}+\expect\left[\normF{\Grad\L(\Theta_{t-1})}\right]
+\eta L\sqrt{\frac{1-\mu_1}{1-\mu_2^t}}\sqrt{\sum_{\tau=1}^\infty\mu_2^{\tau}\tau^2}\\
\le & \frac{\sigma}{\sqrt{b}}+\expect\left[\normF{\Grad\L(\Theta_{t-1})}\right]+ \eta L 
\sqrt{\frac{(1-\mu_1)\mu_2(1+\mu_2)}{(1-\mu_2^t)(1-\mu_2)^3}}\\
\le & \frac{\sigma}{\sqrt{b}}+\expect\left[\normF{\Grad\L(\Theta_{t-1})}\right]+ \sqrt{2}\eta L 
\sqrt{\frac{(1-\mu_1)}{(1-\mu_2^t)(1-\mu_2)^3}}.
\end{align*}
Hence, \begin{equation}\label{eq:dsvt}
\expect\left[\sqrt{\sum_{j=1}^n \hat \bv^j_t}+\frac{n\epsilon}{\sqrt{1-\mu_2^t}}\right]\le \expect\left[\norm{\Grad\L(\Theta_{t-1})}\right]+a_t,
\end{equation} where \begin{equation*}
a_t: = \frac{\sigma}{\sqrt{b}}
+ \sqrt{2}\eta L \sqrt{\frac{(1-\mu_1)}{(1-\mu_2^t)(1-\mu_2)^3}}+\frac{n\epsilon}{\sqrt{1-\mu_2^t}}.
\end{equation*} Then by Lemma~\ref{lem:mutsqrt},\begin{equation}\label{eq:dsctsum}
\frac{1}{T}\sum_{t=1}^T a_t\le \frac{\sigma}{\sqrt{b}}+\left(\sqrt{2}\eta L \sqrt{\frac{(1-\mu_1)}{(1-\mu_2)^3}}+n\epsilon\right)
\left(1-\frac{2\ln(1+\sqrt{1-\mu_2^T})}{T\ln\mu_2}\right).
\end{equation}

\par\noindent\textbf{Step 5: Lower-bounding $\expect\left[d_{t,\max}\normF{\hat M_t}\right]$ and relating to 
$\expect\left[\normF{\Grad\L(\Theta_{t-1})}\right]$ .}
By Cauchy-Schwarz inequality, \begin{align*}
\left(\expect\left[\normF{\hat M_t}\right]\right)^2= & \left(\expect\left[\frac{\normF{\hat M_t}}
{\left(\sqrt{\sum_{j=1}^n \hat \bv^j_t}+n\epsilon_t\right)^{\frac{1}{2}}}
\cdot \left(\sqrt{\sum_{j=1}^n \hat \bv^j_t}+n \epsilon_t\right)^{\frac{1}{2}}\right]\right)^2 \\
\le & \expect\left[\frac{\normF{\hat M_t}^2}
{\sqrt{\sum_{j=1}^n \hat \bv^j_t}+n\epsilon_t}\right]
\expect\left[\sqrt{\sum_{j=1}^n \hat \bv^j_t}+n \epsilon_t\right] \\
\le & \expect\left[\sqrt{\frac{\sum_{j=1}^n\norm{\hat M^j_t}^2}
{\sum_{j=1}^n \hat \bv^j_t+n\epsilon_t}}\cdot\normF{\hat M_t}\right]
\expect\left[\sqrt{\sum_{j=1}^n \hat \bv^j_t}+n \epsilon_t\right]\\
\le & \expect\left[d_{t,\max}\normF{\hat M_t}\right]\expect\left[\sqrt{\sum_{j=1}^n \hat \bv^j_t}+n \epsilon_t\right],
\end{align*}
where $\epsilon_t:=\epsilon/\sqrt{1-\mu_2^t}.$ Combining the above with \eqref{eq:dsvt} gives: \[
\expect\left[d_{t,\max}\normF{\hat M_t}\right]\ge  \frac{\left(\expect\left[\normF{\hat M_t}\right]\right)^2}
{\expect\left[\normF{\Grad\L(\Theta_{t-1})}\right]+a_t}
\ge \frac{\left(\expect\left[\normF{\Grad\L(\Theta_{t-1})}\right]-
\expect\left[\normF{E_t}\right]\right)^2}
{\expect\left[\normF{\Grad\L(\Theta_{t-1})}\right]+a_t}.
\]
Rearranging the terms gives: \[
\left(\expect\left[\normF{\Grad\L(\Theta_{t-1})}\right]\right)^2-
\left(2\expect\left[\normF{E_t}\right]+\expect\left[d_{t,\max}\normF{\hat M_t}\right]\right)
\expect\left[\normF{\Grad\L(\Theta_{t-1})}\right]
-a_t\expect\left[d_{t,\max}\normF{\hat M_t}\right] +\expect\left[\normF{E_t}\right]^2\le 0.
\]
Then solving for $\expect\left[\norm{\Grad\L(\Theta_{t-1})}\right]$ gives \begin{align*}
&\expect\left[\normF{\Grad\L(\Theta_{t-1})}\right]\\
\le & \frac{2\expect\left[\normF{E_t}\right]+\expect\left[d_{t,\max}\normF{\hat M_t}\right]
+\sqrt{\left(2\expect\left[\normF{E_t}\right]+\expect\left[d_{t,\max}\normF{\hat M_t}\right]\right)^2
+4a_t\expect\left[d_{t,\max}\normF{\hat M_t}\right]-4\expect\left[\normF{E_t}\right]^2}}{2}\\
\le & \expect\left[\normF{E_t}\right]+\frac{1}{2}\expect\left[d_{t,\max}\normF{\hat M_t}\right]
+\sqrt{\frac{1}{4}\expect\left[d_{t,\max}\normF{\hat M_t}\right]^2
+\left(\expect\left[\normF{E_t}\right]+a_t\right)\expect\left[d_{t,\max}\normF{\hat M_t}\right]}\\
\le & \expect\left[\normF{E_t}\right]+\expect\left[d_{t,\max}\normF{\hat M_t}\right]
+\sqrt{\left(\expect\left[\normF{E_t}\right]+a_t\right)\expect\left[d_{t,\max}\normF{\hat M_t}\right]}\\
\le & \expect\left[\normF{E_t}\right]+\expect\left[d_{t,\max}\normF{\hat M_t}\right]
+\sqrt{\left(\expect\left[\normF{E_t}\right]+a_t\right)\expect\left[d_{t,\max}\normF{\hat M_t}\right]}.
\end{align*} 

\par\noindent\textbf{Step 6: Deriving the convergence rate}
By Cauchy-Schwarz inequality and \eqref{eq:dcond}, it then follows that: \begin{align*}
\frac{1}{T}\sum_{t=1}^T\expect\left[\normF{\Grad\L(\Theta_{t-1})}\right]  
\le & \frac{1}{T}\sum_{t=1}^T\expect\left[\normF{E_t}\right]+
\frac{1}{T}\min\left\{\kappa_t,\frac{1}{c^2}\right\}\sum_{t=1}^T\expect\left[d_{t,\min}\normF{\hat M_t}\right]\\
& + \sqrt{\frac{1}{T}\sum_{t=1}^T\left(\expect\left[\normF{E_t}\right]+a_t\right)}
\sqrt{\frac{1}{T}\min\left\{\kappa_t,\frac{1}{c^2}\right\}\sum_{t=1}^T\expect\left[d_{t,\max}\normF{\hat M_t}\right]}.
\end{align*}
Combining the above with \eqref{eq:dtotalM} and \eqref{eq:dsctsum} gives:
\begin{align*}
&\frac{1}{T}\sum_{t=1}^T\expect\left[\normF{\Grad\L(\Theta_{t-1})}\right]\\
\le &\frac{1}{T}\sum_{t=1}^T\expect\left[\norm{E_t}_*\right]+  \frac{\Delta}{\eta Tc^2}+\frac{\eta L (1-\mu_1)}{2c^2 (1-\mu_2)}
+ \frac{1}{Tc^2}\sqrt{\frac{1-\mu_1}{1-\mu_2}}\sum_{t=1}^T\expect\left[\norm{E_t}_*\right] \\
& +\sqrt{\frac{1}{T}\sum_{t=1}^T\left(\expect\left[\normF{E_t}\right]+a_t\right)} \cdot \sqrt{\frac{\Delta}{\eta Tc^2}
+ \frac{1}{Tc^2}\sqrt{\frac{1-\mu_1}{1-\mu_2}}\sum_{t=1}^T\expect\left[\norm{E_t}_*\right] 
+\frac{\eta L (1-\mu_1)}{2c^2 (1-\mu_2)}}\\
\le & \frac{1}{T}\sum_{t=1}^T\expect\left[\norm{E_t}_*\right]+  \frac{\Delta}{\eta Tc^2}+\frac{\eta L (1-\mu_1)}{2c^2 (1-\mu_2)}
+ \frac{1}{Tc^2}\sqrt{\frac{1-\mu_1}{1-\mu_2}}\sum_{t=1}^T\expect\left[\norm{E_t}_*\right] \\
& +\sqrt{\frac{1}{T}\sum_{t=1}^T \expect\left[\normF{E_t}\right]+
\frac{\sigma}{\sqrt{b}}+\left(\sqrt{2}\eta L \sqrt{\frac{(1-\mu_1)}{(1-\mu_2)^3}}+n\epsilon\right)
\left(1-\frac{2\ln(1+\sqrt{1-\mu_2^T})}{T\ln\mu_2}\right)}\\
& \cdot \sqrt{\frac{\Delta}{\eta Tc^2}
+ \frac{1}{Tc^2}\sqrt{\frac{1-\mu_1}{1-\mu_2}}\sum_{t=1}^T\expect\left[\norm{E_t}_*\right] 
+\frac{\eta L (1-\mu_1)}{2c^2 (1-\mu_2)}}.
\end{align*}
In particular, for large $T>0,$ if choosing $\eta=\bigO(T^{-\frac{3}{4}})$, $1-\mu_1=\Theta(T^{-\frac{1}{2}})$,
$1-\mu_2=\Theta(T^{-\frac{1}{2}}),$ $\epsilon =\bigO(T^{-\frac{1}{2}})$, and $c=\Theta(1)$, then, 
by \eqref{eq:dsumEt},\[
\frac{1}{T}\sum_{t=1}^T\expect\left[\norm{E_t}_*\right]
\le\bigO\left(T^{-\frac{1}{4}}\right),
\]
and it follows that:
\[
\frac{1}{T}\sum_{t=1}^T\expect\left[\norm{\Grad\L(\Theta_{t-1})}\right]\le \bigO\left(T^{-\frac{1}{4}}+\sqrt{\sigma}b^{-\frac{1}{4}}T^{-\frac{1}{8}}\right),
\] where $b$ is the batch size. 
The proof is thus completed.
\end{proof}
\end{document}